\documentclass[twoside]{article}

\usepackage[preprint]{aistats2026}
\usepackage[round]{natbib}

\usepackage[utf8]{inputenc} 
\usepackage[T1]{fontenc}    
\usepackage{hyperref}       
\usepackage{url}            
\usepackage{booktabs}       
\usepackage{amsfonts}       
\usepackage{nicefrac}       
\usepackage{microtype}      
\usepackage{xcolor}         

\usepackage{amsmath}
\usepackage{amssymb}
\usepackage{mathtools}
\usepackage{amsthm}

\usepackage{algorithm}
\usepackage{algorithmic}
\usepackage{amsthm}
\usepackage{subfigure}
\usepackage{epsfig}
\usepackage{graphicx}
\usepackage{url}
\usepackage{multirow}
\usepackage{amssymb}
\usepackage{amsthm}
\usepackage{mathrsfs}
\usepackage{epstopdf}
\usepackage{multicol}
\usepackage{amsmath}
\usepackage{bm}
\usepackage{multibib}
\usepackage{xcolor}
\usepackage{multicol}
\usepackage{wrapfig}
\usepackage{boxedminipage}

\usepackage[most]{tcolorbox}
\definecolor{lightgray}{gray}{0.9}

\allowdisplaybreaks

\newtheorem{theorem}{\textbf{Theorem}}
\newtheorem{assumption}{\textbf{Assumption}}[section]
\newtheorem{lemma}{\textbf{Lemma}}[section]

\newtheorem{remark}{\textbf{Remark}}[section]

\begin{document}

%

%

\twocolumn[

\aistatstitle{Federated Stochastic Minimax Optimization under Heavy-Tailed Noises}

\aistatsauthor{ Xinwen Zhang \And Hongchang Gao }

\aistatsaddress{ Temple University \And  Temple University } ]

\begin{abstract}
Heavy-tailed noise has attracted growing attention in nonconvex stochastic optimization, as numerous empirical studies suggest it offers a more realistic assumption than standard bounded variance assumption. In this work, we investigate nonconvex–PL minimax optimization under heavy-tailed gradient noise in federated learning.
We propose two novel algorithms: Fed-NSGDA-M, which integrates normalized gradients, and FedMuon-DA, which leverages the Muon optimizer for local updates. Both algorithms are designed to effectively address heavy-tailed noise in federated minimax optimization, under a milder condition. We theoretically establish that both algorithms achieve a convergence rate of $O({1}/{(TNp)^{\frac{s-1}{2s}}})$. To the best of our knowledge, these are the first federated minimax optimization algorithms with rigorous theoretical guarantees under heavy-tailed noise. Extensive experiments further validate their effectiveness.
\end{abstract}

\section{Introduction}
In this paper, we study the problem of federated stochastic minimax optimization under heavy-tailed gradient noise:
\begin{align}\label{eq:loss}
	\min_{x\in \mathbb{R}^{d_x}} \max_{y\in \mathbb{R}^{d_y}} f(x, y) \triangleq \frac{1}{N}\sum_{n=1}^{N}f^{(n)}(x, y) \ , 
\end{align}
where $f^{(n)}(x, y) = \mathbb{E}[f^{(n)}(x, y; \xi^{(n)})]$ denotes the expected loss on the $n$-th client, $n \in \{1, \cdots ,N\}$, and $\xi^{(n)}$ represents a random sample on client $n$. We assume that $f(x,y)$ is nonconvex with respect to variable $x$ and satisfies the Polyak-Lojasiewicz (PL) condition~\citep{POLYAK1963864} with respect to variable $y$. 

To meet the demands of large-scale training under privacy constraints, federated learning~\citep{mcmahan2017communication} has emerged as a prominent distributed framework that enables multiple local clients to train a global model without sharing raw data. In particular, federated stochastic minimax optimization has attracted significant attention in machine learning community due to its broad range of applications, including generative adversarial networks~\citep{arjovsky2017wasserstein}, fair classification~\citep{nouiehed2019solving},  adversarially robust learning ~\citep{reisizadeh2020robust}, and deep AUC maximization~\citep{guo2020communication}. To address these applications in the federated learning setting, a variety of algorithms have been developed~\citep{deng2021local, sharma2022federated, yang2022sagda, wu2023solving, shen2024stochastic}. Nevertheless, a common limitation is that their theoretical analyses rely on the standard finite-variance assumption, i.e., the second moment of the difference between the stochastic gradient and the true gradient is bounded.

However, recent studies \citep{simsekli2019tail, zhang2020adaptive, gurbuzbalaban2021heavy, barsbey2021heavy} reveal a more realistic picture: the stochastic gradient of many modern machine learning models follows a \textbf{heavy-tailed} distribution.
In heavy-tailed regimes, the variance can be unbounded, and a single extremely large gradient can dominate the update, leading to instability in training. In federated learning, such effects are further amplified by data heterogeneity~\citep{charles2021large, yang2022taming}, causing many existing algorithms to struggle under heavy-tailed noise. These findings challenge the validity of traditional methods  built on standard assumptions, as they raise serious convergence concerns and may even lead to divergence.

To address the challenges introduced by heavy-tailed noises, several works~\citep{gorbunov2020stochastic, zhang2020gradient, cutkosky2021high} proposed the use of clipping technique, which discards outlier gradients beyond a threshold $\tau$ to ensure convergence in the single-machine setting. Building on this idea, subsequent studies~\citep{yang2022taming, lee2025efficient} extended clipping to federated learning and provided corresponding convergence analyses. More recently, \cite{hubler2025gradient, liu2025nonconvex} have identified key limitations of clipping, including difficulties of tuning the hyperparameter $\tau$ and misalignment between theoretical guarantees and empirical performance. As an alternative, they introduced gradient normalization without clipping to better handle heavy-tailed noises. Despite these advances, existing work has largely focused on stochastic \textit{minimization} problems, leaving stochastic \textit{minimax} optimization under heavy-tailed noise largely unexplored. Importantly, minimax formulations introduce additional dual variables, which significantly complicate the optimization compared to standard minimization. Consequently, it remains an open question whether normalized gradients can guarantee convergence in the minimax setting, and if so, what convergence rates can be achieved in the presence of dual variables? 

Furthermore, current algorithms with normalized gradients for heavy-tailed noise~\citep{hubler2025gradient, liu2025nonconvex} are limited to the single-machine setting and overlook the additional challenges that arise in federated learning scenarios, such as data heterogeneity. In conventional federated minimax optimization, techniques such as control variates~\citep{karimireddy2020scaffold, yang2022sagda} have been widely employed to address heterogeneity. However, their effectiveness under heavy-tailed noise remains unexplored. In particular, it is still unclear whether these techniques can mitigate data heterogeneity in the heavy-tailed setting.

These gaps naturally motivate the following question: \textbf{Is it possible to design a provably efficient algorithm for federated stochastic minimax optimization under heavy-tailed noises that better adapts to practical applications?}

Besides normalization, another newly introduced optimizer, Muon~\citep{jordan2024muon}, has drawn our attention. Recent studies~\citep{liu2025muon, shah2025practical} show that Muon delivers strong performance in training both small- and large-scale language models. Note that language data are intrinsically heavy-tailed~\citep{kunstner2024heavy}, which suggests that Muon may also be suitable for handling heavy-tailed noise. However, no existing analysis provides theoretical support for this hypothesis. To the best of our knowledge, there is only one study~\citep{sfyraki2025lions} that investigates Muon in the heavy-tailed regime, but it considers a single-machine setting and its algorithm and analysis rely on clipping. This naturally leads to the following question:
\textbf{Is it possible for the Muon optimizer itself to effectively address federated stochastic minimax optimization under heavy-tailed noise with provable guarantees?}

\subsection{Our Contributions}
We address these questions through both novel algorithmic design and rigorous theoretical analysis. The main contributions of this paper are summarized below:
\begin{itemize}
    \item We propose a novel algorithm, Fed-NSGDA-M, which incorporates normalized gradients and control variates into federated stochastic minimax optimization to solve Eq.~(\ref{eq:loss}) under heavy-tailed noise. To the best of our knowledge, heavy-tailed noise has not been investigated in the context of federated minimax problems, and thus our method represents the first algorithm developed for this setting. In addition, Fed-NSGDA-M effectively handles data heterogeneity without requiring any heterogeneity bounds across clients.
    \item We propose another algorithm named FedMuon-DA, which employs the Muon optimizer to update local variables to solve Eq.~(\ref{eq:loss}). To the best of our knowledge, this is the first work to study Muon in the context of federated minimax problems. Moreover, FedMuon-DA also provides the ability to handle heavy-tailed noise and data heterogeneity without requiring heterogeneity bounds and clipping operation.
    \item We provide a rigorous convergence analysis of both Fed-NSGDA-M and FedMuon-DA for  federated nonconvex-PL minimax problems. In particular, we show that both methods achieve a convergence rate of  $O({1}/{(TNp)^{\frac{s-1}{2s}}})$ under heavy-tailed noises.
    \item We conducted extensive experiments for text classification in both homogeneous and heterogeneous federated settings under heavy-tailed noise. The results demonstrate that both Fed-NSGDA-M and FedMuon-DA outperform existing baselines, validating its effectiveness in practice.
\end{itemize}

\section{Related Work}
\subsection{Heavy-Tailed Noises}
Recent studies~\citep{zhang2020adaptive, simsekli2019tail, gurbuzbalaban2021heavy, barsbey2021heavy, battash2024revisiting} have shown that heavy-tailed noises naturally arises when training deep neural networks, including language models and attention-based architectures~\citep{zhang2020adaptive, kunstner2024heavy, ahn2024linear}. Traditional SGD methods may diverge under heavy-tailed noises, and gradient clipping has been proposed as an effective technique to ensure convergence~\citep{gorbunov2020stochastic, zhang2020adaptive, cutkosky2021high, liu2023breaking}.   More recently, the difficulty of tuning clipping threshold has motivated the development of gradient normalization as a more robust alternative~\citep{sun2024gradient, hubler2025gradient, liu2025nonconvex}. In particular, \cite{hubler2025gradient} investigated the limitations of clipping and established the convergence rate of normalized SGD under heavy-tailed noise, while \cite{liu2025nonconvex} derived the convergence rate of a normalized momentum algorithm. Similarly, \cite{sun2024gradient} showed that gradient normalization alone is sufficient to guarantee convergence, though their analysis relies on a stronger assumption of individual lipschitzness. In federated learning, heavy-tailed noises naturally emerges from data heterogeneity~\citep{charles2021large, yang2022taming}, often leading to catastrophic training failures. So far, research on federated learning under heavy-tailed noise has been limited to clipping techniques~\citep{yang2022taming, lee2025efficient}. To the best of our knowledge, normalized gradients have not yet been explored in this setting.

\subsection{Muon}
\cite{jordan2024muon} first introduced Muon as an orthonormalized optimizer for training neural network hidden layers, and \cite{bernstein2024old} characterized its update rule as  performing steepest descent under a spectral norm constraint. More recently, Muon has demonstrated practical efficiency for language models~\citep{liu2025muon, shah2025practical}, and a growing line of work has further investigated its convergence properties~\citep{li2025note, an2025asgo, kovalev2025understanding, shen2025convergence}. However, none of these studies have analyzed the convergence behavior of Muon under heavy-tailed noise, a more realistic assumption in modern machine learning. \cite{sfyraki2025lions} is the only work that considers Muon in the stochastic frank-wolfe method under heavy-tailed noise, but their analysis relies on gradient clipping. Moreover, existing studies mainly focuses on single-machine settings, leaving the convergence of Muon in federated heavy-tailed settings entirely unexplored.

\subsection{Federated Minimax Optimization}
Federated minimax optimization has emerged as a central topic in distributed machine learning, motivated by its central role in large-scale training and its broad range of applications~\citep{deng2020distributionally, reisizadeh2020robust, rasouli2020fedgan, beznosikov2025distributed}. Early progress was made by~\cite{deng2021local}, who introduced the LocalSGDA framework, where each client performs multiple local updates before synchronizing, and established convergence guarantees for general federated minimax problems. Building on this foundation,~\cite{sharma2022federated} employed momentum to LocalSGDA, demonstrating linear speedup with the number of clients. To address data heterogeneity,~\cite{yang2022sagda} proposed the SAGDA algorithm, which leverages stochastic sampling and control variates~\citep{karimireddy2020scaffold} without requiring bounded heterogeneity assumptions. Subsequently,~\cite{wu2023solving} achieved improved convergence rates by applying the STORM gradient estimator~\citep{cutkosky2019momentum}, and~\cite{shen2024stochastic} explored smoothing techniques~\citep{yang2022faster} in the federated minimax setting. However, all the aforementioned works rely on the standard finite-variance noise assumption, and their algorithmic designs and theoretical analyses are not effective under heavy-tailed noises. Moreover, to the best of our knowledge, no existing methods, whether based on clipping or normalization, have been developed for federated minimax problems under this assumption.

\begin{algorithm*}[ht]
	\caption{Fed-NSGDA-M}
	\label{alg}
    \small
	\begin{algorithmic}[1]
		\REQUIRE initial model $x_0$, $y_0$, global learning rates $\gamma_x$, $\gamma_y$, local learning rates $\eta_x$, $\eta_y$, momentum parameter $\beta_x$, $\beta_y$, local updates rounds $P$, and communication rounds $T$.  \\
		\vspace{2mm}
		\FOR{$t=0,\cdots, T-1$} 		
		\begin{tcolorbox}[colback=lightgray,colframe=black!40,boxrule=0.3pt,arc=2pt,left=2pt,right=2pt,top=1pt,bottom=1pt]
        \FOR{each client $n$}
		\STATE Initialize local model $x^{(n)}_{t,0} = x_{t}$, $y^{(n)}_{t,0} = y_{t}$.

        \FOR{$i=0,\cdots, p-1$}
		\STATE  Compute local momentum: \\
        $\ u^{(n)}_{t,i} = \beta_x(\nabla_x f^{(n)}(x^{(n)}_{t,i}, y^{(n)}_{t,i};\xi^{(n)}_{t,i}) +   g_{x,t-1} - g^{(n)}_{x,t-1}) + (1-\beta_x)u_{t-1}$ \ , \\
        $\ v^{(n)}_{t,i} = \beta_y(\nabla_y f^{(n)}(x^{(n)}_{t,i}, y^{(n)}_{t,i};\xi^{(n)}_{t,i}) +  g_{y,t-1} - g^{(n)}_{y,t-1}) + (1-\beta_y)v_{t-1}$ \ .
        \STATE \textcolor{blue}{Normalized local update:} 
        $ \ x^{(n)}_{t,i+1} = x^{(n)}_{t,i} - \eta_x \frac{u^{(n)}_{t,i}}{\|u^{(n)}_{t,i}\|}$\ , 
        $\ y^{(n)}_{t,i+1} = y^{(n)}_{t,i} + \eta_y \frac{v^{(n)}_{t,i}}{\|v^{(n)}_{t,i}\|}$ \ . 
		\ENDFOR
        \STATE Aggregate local control variates: \\
        $ \  g_{x,t}^{(n)} = \frac{1}{p}\sum_{i=0}^{p-1}\nabla_x f^{(n)}(x^{(n)}_{t,i}, y^{(n)}_{t,i}; \xi^{(n)}_{t,i})$ \ , 
        $\  g_{y,t}^{(n)} = \frac{1}{p}\sum_{i=0}^{p-1}\nabla_y f^{(n)}(x^{(n)}_{t,i}, y^{(n)}_{t,i}; \xi^{(n)}_{t,i})$ \ .
		\ENDFOR
        \end{tcolorbox}
        \textbf{Central Server}: \\
        \begin{tcolorbox}[colback=lightgray,colframe=black!40,boxrule=0.3pt,arc=2pt,left=2pt,right=2pt,top=1pt,bottom=1pt]
        \STATE Aggregate global control variates: 
        $\ g_{x,t} = \frac{1}{N}\sum_{n=1}^Ng_{x,t}^{(n)}$ \ , 
        $\ g_{y,t} = \frac{1}{N}\sum_{n=1}^Ng_{y,t}^{(n)}$ \ .
        \STATE Global update: 
        $\ x_{t+1} = x_{t} + \frac{\gamma_x}{\eta_x Np}\sum_{n=1}^N( x^{(n)}_{t,p} - x_{t})$ \ , 
        $\  y_{t+1} = y_{t} + \frac{\gamma_y}{\eta_y Np}\sum_{n=1}^N(y^{(n)}_{t,p} - y_{t} )$ \ .
        \STATE Update global momentum: 
        $\  u_{t}  =   \beta_x g_{x,t} + (1-\beta_x)u_{t-1} $ \ ,  
        $\ v_{t}  = \beta_y g_{y,t} + (1-\beta_y)v_{t-1} $ \ . 
        \end{tcolorbox}
		\ENDFOR
	\end{algorithmic}
\end{algorithm*}

\vspace{-5pt}
\section{Algorithm}
\vspace{-5pt}
\subsection{Assumptions}
To solve Eq.~(\ref{eq:loss}), we introduce some commonly used assumptions in the federated minimax optimization~\citep{sharma2022federated, wu2023solving, shen2024stochastic}. 
\begin{assumption}\label{assumption:smooth}
(\textbf{Smoothness}) For any $n \in\{1, 2, \cdots, N\}$, $\nabla f^{(n)}(\cdot, \cdot)$ is $L_f$-Lipschitz continuous, where $L_f>0$.
\end{assumption}

\begin{assumption}\label{assumption:pl}
(\textbf{PL condition}) For any fixed $x\in \mathbb{R}^{d_x}$, $\max_{y\in\mathbb{R}^{d_y}}f(x, y)$, has a nonempty solution set and a finite optimal value. There exists $\mu>0$ such that $\|\nabla_y f(x, y)\|^2 \geq 2\mu (f(x, y^*(x)) - f(x, y))$,  where $y^*(x) =\arg\max_{y\in\mathbb{R}^{q}} f(x, y)$. 
\end{assumption}

We also introduce the assumption of heavy-tailed noises~\citep{yang2022taming, hubler2025gradient, lee2025efficient}.
\begin{assumption}\label{assumption:ht_variance}
(\textbf{Heavy-Tailed Noises})	For any $n \in\{1, 2, \cdots, N\}$, the gradients of each function $f^{(n)}(x,y)$ are unbiased.  Moreover,  there exist $s\in(1, 2]$ and $\sigma>0$ such that
	$\mathbb{E}[\|\nabla f^{(n)}(x,y; \xi)-\nabla f^{(n)}(x,y)\|^s] \leq \sigma^{s}$. 
\end{assumption}
This assumption is weaker than the standard bounded variance assumption, which is recovered as a special case when $s=2$.

Note that most existing approaches for federated minimax optimization~\citep{sharma2022federated, wu2023solving, shen2024stochastic}, rely on the assumption of bounded heterogeneity:
\begin{align}
    \|\nabla f^{(n)}(x, y) - \nabla f(x, y) \|^2 \leq \delta^2 \ , \notag
\end{align}
where $\delta>0$. In this work, we remove this requirement and show that our method can handle data heterogeneity without assuming bounded heterogeneity, thereby operating under a strictly milder condition.

\vspace{-5pt}
\subsection{Matrix Variant}
\vspace{-5pt}
In this paper, we also consider the setting where both variables are in matrix form, as shown below:\begin{align}\label{eq_muon:loss}
	\min_{X\in \mathbb{R}^{m_x\times n_x}} \max_{y\in \mathbb{R}^{m_y\times n_y}} f(X, Y) \triangleq \frac{1}{N}\sum_{n=1}^{N}f^{(n)}(X, Y) \ .
\end{align}
Under this setting, we assume Assumptions~\ref{assumption:smooth}-\ref{assumption:ht_variance} also hold for matrices $X$ and $Y$.

\paragraph{Notation.} We u denote the condition number by $\kappa=L_f/\mu$. Since Muon is an optimizer designed directly for matrices, we introduce the following matrix notations. For a matrix $X\in \mathbb{R}^{m\times n}$, we use $\|X\|_F$ to denote the Frobenius norm, $\|X\|_*$ the nuclear norm, and $\|X\|_2$ the spectral norm. 

\begin{algorithm*}[ht]
	\caption{FedMuon-DA: with Muon local update}
	\label{alg:fedmuon_update}
    \small
	\begin{algorithmic}[1]
    \begin{tcolorbox}[colback=lightgray,colframe=black!40,boxrule=0.4pt,arc=2pt,left=0pt,right=0pt,top=1pt,bottom=1pt]
	   \STATE  $\ $ Orthonormalize $U^{(n)}_{t,i}$ with Newton–Schulz approach: $(P^{(n)}_{t,i}, \Sigma^{(n)}_{t,i}, Q^{(n)}_{t,i})=\text{SVD}(U^{(n)}_{t,i})$ \ , 
	   \STATE $\ $ Update variable $X^{(n)}_{t,i}$: 	$X^{(n)}_{t+1,i} =  X^{(n)}_{t,i} -\eta_{x} P^{(n)}_{t,i}(Q^{(n)}_{t,i})^T$ \ , 
       \STATE $\ $ Orthonormalize $V^{(n)}_{t,i}$ with Newton–Schulz approach:  $(R^{(n)}_{t,i}, \Sigma^{(n)}_{t,i}, S^{(n)}_{t,i})=\text{SVD}(V^{(n)}_{t,i})$ \ , 
	   \STATE $\ $ Update variable $Y^{(n)}_{t,i}$: $Y^{(n)}_{t+1,i} =  Y^{(n)}_{t,i} +\eta_{y} R^{(n)}_{t,i}(S^{(n)}_{t,i})^T$ .
    \end{tcolorbox}
	\end{algorithmic}
\end{algorithm*}

\vspace{-5pt}
\subsection{Fed-NSGDA-M}
\vspace{-5pt}
To solve Eq.~(\ref{eq:loss}) under heavy-tailed noises, we propose Fed-NSGDA-M, as outlined in Algorithm~\ref{alg}. For each client $n$,  the local momentum is computed in Step 5:
\begin{align}
    & u^{(n)}_{t,i} = (1-\beta_x)u_{t-1}  \\
    & \quad \quad + \beta_x(\nabla_x f^{(n)}(x^{(n)}_{t,i}, y^{(n)}_{t,i};\xi^{(n)}_{t,i}) +  g_{x,t-1} - g^{(n)}_{x,t-1}) \ , \notag 
\end{align}
where $0<\beta_x<1$,   $g_{x,t-1}$ and $g^{(n)}_{x,t-1}$ denote the global and local control variates for the primal variable, and their difference helps mitigate the impact of data heterogeneity. Moreover, $u_{t-1}$ is the global momentum updates in Step 12:
\begin{align}
   & u_{t} = (1-\beta_x)u_{t-1} +  \beta_x g_{x,t} \ . 
\end{align}
The dual variable $y$ is updated in the same manner, based on the local momentum $v^{(n)}_{t,i}$, the global momentum $v_{t}$, and the global and local control variates $g_{y,t-1}$ and $g^{(n)}_{y,t-1}$.

Subsequently, we update $x$ and $y$ locally in Step 6 using normalized momentum:
\vspace{-5pt}
\begin{align}
    &  x^{(n)}_{t,i+1} = x^{(n)}_{t,i} - \eta_x \frac{u^{(n)}_{t,i}}{\|u^{(n)}_{t,i}\|} \ ,  \notag \\
    & y^{(n)}_{t,i+1} = y^{(n)}_{t,i} + \eta_y \frac{v^{(n)}_{t,i}}{\|v^{(n)}_{t,i}\|} \ ,
\end{align}
where $\eta_x>0$ and $\eta_y>0$ are the local learning rates.  

After every $p$ local iterations, the server performs a communication step and updates the global model in Step 11:
\vspace{-5pt}
\begin{align}
    &  x_{t+1} = x_{t} + \frac{\gamma_x}{\eta_x Np}\sum_{n=1}^N( x^{(n)}_{t,p} - x_{t}) \ , \notag \\
    &  y_{t+1} = y_{t} + \frac{\gamma_y}{\eta_y Np}\sum_{n=1}^N(y^{(n)}_{t,p} - y_{t} ) \ , 
\end{align}
where $\gamma_x>0$ and $\gamma_y>0$ are the global learning rates. 

In addition, the local and global control variates are aggragated in Step 8 and 10.

The design of Algorithm~\ref{alg} benefits from the following aspects: 1) \textbf{Local gradient normalization}, which effectively handles heavy-tailed gradient noise without the need to carefully tuning a clipping threshold hyperparameter $\tau$, thereby stabilizes the learning process; 2) \textbf{Control variates}, which correct client drift and mitigate the adverse effects of heterogeneous data distributions.

\subsection{FedMuon-DA}
We further propose FedMuon-DA (see Algorithm~\ref{alg-FedMuon-DA} in the appendix) by 
replacing the normalized gradient step in Algorithm~\ref{alg} (Step 6) with the Muon update, whose procedure is detailed in Algorithm~\ref{alg:fedmuon_update}.  
Specifically, Muon orthonormalizes the momentum $U^{(n)}_{t,i} \in \mathbb{R}^{m_x \times n_x}$ via the following problem:
\begin{equation}
	O = \arg\min_{O} \|O- U \|^2_F, \quad s.t. \quad O^TO=I_{n} \ , 
\end{equation}
where $I_{n}\in\mathbb{R}^{n_x\times n_x}$ is the identity matrix.  The optimal solution is given by $O=PQ^T$ with $P \in \mathbb{R}^{m_x\times r_x}$ and $Q\in \mathbb{R}^{n_x\times r_x}$ obtained from the singular value decomposition (SVD) of $U$, i.e., $U=P\Sigma Q^T$. Here,  $\Sigma \in \mathbb{R}^{r_x\times r_x}$ is a diagonal matrix containing the singular values of $U$, and $r$ represents the rank of $U$. With this orthonormalization step, FedMuon-DA updates the local variable $X^{(n)}_{t,i}$ as follows:
\begin{align}
    X^{(n)}_{t+1,i} =  X^{(n)}_{t,i} -\eta_{x} P^{(n)}_{t,i}(Q^{(n)}_{t,i})^T \ , 
\end{align}
and the same procedure is applied for the dual momentum $V^{(n)}_{t,i} \in  \mathbb{R}^{m_y \times n_y}$ and dual variable $Y^{(n)}_{t,i} \in  \mathbb{R}^{m_y \times n_y}$.

\section{Convergence Analysis}
To establish the convergence rate of our algorithm, we introduce the following auxiliary function:
 \begin{equation}
     \Phi(x) = f(x,y^*(x)) = \max_{y\in \mathbb{R}^{d_y}} f(x,y)  \ . 
\end{equation}
Therefore, $\Phi$ is $L_{\Phi}$-smooth, where $L_{\Phi} = L_f + \frac{L_f^2}{\mu}$~\citep{nouiehed2019solving}.
In terms of these auxiliary functions, we obtain:
\begin{equation}
    \begin{aligned}
        \min_{x\in \mathbb{R}^{d_x}} \max_{y\in \mathbb{R}^{d_y}} f(x, y) = \min_{x\in \mathbb{R}^{d_x}} \Phi(x) \ . \\
    \end{aligned}
\end{equation}

\subsection{Convergence Rate of Algorithm~\ref{alg}}
Based on the introduced auxiliary function and Assumption~\ref{assumption:smooth}-\ref{assumption:ht_variance}, we establish the convergence rate of Algorithm~\ref{alg}. 

\begin{theorem}\label{theorem:convergence-rate}
	Given Assumptions~\ref{assumption:smooth}-\ref{assumption:ht_variance},
    by setting
    \begin{align}
        & \gamma_{x} = O\left(\frac{(Np)^{1/4}}{\kappa T^{3/4}}\right) \ , \quad \gamma_y = O(\kappa\gamma_{x}) \ ,  \notag \\
        & \beta_{x} = O\left(\frac{(Np)^{1/2}}{T^{1/2}}\right) \ ,  \quad \beta_y = O(\beta_x) \ , \\
        & \eta_x = O\left(\frac{1}{p\sqrt{T}}\right) \ , \quad \eta_y =   O(\eta_{x}) \ , \notag 
    \end{align}
    we obtain
    \begin{align}
        \frac{1}{T}\sum_{t=0}^{T-1}\mathbb{E}[\| \nabla \Phi({x}_{t})  \|] & \leq O\left(\frac{\kappa}{(TNp)^{1/4}}+ \frac{\kappa \sigma}{(TNp)^{\frac{s-1}{2s}}}\right) \ . 
    \end{align}
\end{theorem}

\begin{remark} 
(\textbf{Convergence rate}) As $s\in (1, 2]$,  the second term in the convergence upper bound dominates the first term. The convergence rate of Fed-NSGDA-M is $O\left(\frac{1}{(TNp)^{\frac{s-1}{2s}}}\right)$, which implies a linear speedup with respect to the number of clients $N$. In the special case $N=1$, the result matches the convergence rate established for the single-machine algorithm under heavy-tailed noises in the minimization formulation ~\citep{liu2025nonconvex, hubler2025gradient}. When $s=2$,  the heavy-tailed noise assumption reduces to the standard bounded variance case. In this regime,  Fed-NSGDA-M achieves a convergence rate of $O\left(\frac{1}{(TNp)^{1/4}}\right)$, which matches the result of LocalSGDAM~\citep{sharma2022federated}.
\end{remark}

\begin{remark}
   (\textbf{Communication complexity}) For Theorem~\ref{theorem:convergence-rate}, by setting $Np=O(T^{\frac{1}{3}})$, we have 
   \begin{align}
        \frac{1}{T}\sum_{t=0}^{T-1}\mathbb{E}[\| \nabla \Phi({x}_{t})  \|] & \leq O\left(\frac{\kappa}{T^{1/3}}+ \frac{\kappa \sigma}{T^{\frac{2(s-1)}{3s}}}\right) \ . 
    \end{align}
    Then, to achieve the $\epsilon$-accuracy solution, $ \frac{1}{T}\sum_{t=0}^{T-1}\mathbb{E}[\| \nabla \Phi({x}_{t})  \|]\leq\epsilon$, the communication complexity is $T=O\left(\left(\frac{\kappa}{\epsilon}\right)^{\frac{3s}{2(s-1)}}\right)$, as $s\in (1, 2]$. When $s=2$, we have $T=O\left(\frac{\kappa^{3}}{\epsilon^{3}}\right)$, which matches the communication complexity of LocalSGDAM~\citep{sharma2022federated} in terms of both $\epsilon$ and $\kappa$.
\end{remark}

\begin{remark}
    (\textbf{Hyperparameter.}) In Theorem~\ref{theorem:convergence-rate}, the ratio between two learning rates is $\gamma_x/\gamma_y=O(1/\kappa)$. In contrast, existing methods, such as LocalSGDAM~\citep{sharma2022federated}, have $\gamma_x/\gamma_y=O(1/\kappa^2)$, which means that the learning rate of $x$ should be much smaller than that of $y$ in LocalSGDAM, while our two learning rates much more balanced than LocalSGDAM. 
\end{remark}

\paragraph{Sketch of the Proof of Theorem~\ref{theorem:convergence-rate}}
Our theoretical analysis relies on the following potential function:
\begin{align}
        \mathcal{L}_{t} = 3\mathbb{E}[\Phi(x_{t})]  + (\mathbb{E}[\Phi(x_{t})] -  \mathbb{E}[f(x_{t}, y_{t})] ) \ . 
\end{align}
We first establish the descent property of $\mathbb{E}[\Phi(x_{t})]$ in Lemma~\ref{lemma:phi_smooth}, and that of $\mathbb{E}[\Phi(x_{t})] -  \mathbb{E}[f(x_{t}, y_{t})]$ in Lemma~\ref{lemma:phi_f}. Building on these results, by setting $\gamma_{x} = \frac{\gamma_{y}}{10\kappa}$, Lemma~\ref{lemma:phi_sum} further derives an upper bound for $\frac{1}{T}\sum_{t=0}^{T-1}\mathbb{E}[\| \nabla \Phi({x}_{t})  \|]$ within the above potential function, which serves as our core lemma:
\begin{lemma}\label{lemma_1}
    Given Assumptions~\ref{assumption:smooth}-\ref{assumption:ht_variance}, by setting $\gamma_{x} = \frac{\gamma_{y}}{10\kappa}$, the following inequality holds:
    \begin{align}\label{eq:phi_sum_1}
        & \frac{1}{T}\sum_{t=0}^{T-1}\mathbb{E}[\| \nabla \Phi({x}_{t})  \|]  \leq \frac{(\Phi(x_{0})- \Phi^*) }{\gamma_{x}T} + \frac{\Phi(x_{0}) - f(x_0,y_0)}{3\gamma_{x}T}\notag \\
        &  + \frac{10}{3}\underbrace{\frac{1}{T}\sum_{t=0}^{T-1}\mathbb{E}[\|\nabla_{x} f(x_{t}, y_{t}) - u_{t} \|]}_{\text{gradient estimation error for x}}  + \frac{2L_{\Phi}\gamma_{x}}{3}  \notag \\
        &  +  \frac{20\kappa}{3}\underbrace{\frac{1}{T}\sum_{t=0}^{T-1}\mathbb{E}[\| \nabla_y f(x_{t}, y_{t}) - v_{t}\|]}_{\text{gradient estimation error for y}}   + \frac{L_f\gamma_{x}(1+10\kappa)^2}{6}\notag \\
        &  + \frac{5}{3}\underbrace{\frac{1}{NpT}\sum_{t=0}^{T-1}\sum_{n=1}^N\sum_{i=0}^{p-1} \mathbb{E}[ \| u_{t} - u_{t,i}^{(n)}\| ]}_{\text{consensus error for x}}   \notag \\
        &  + \frac{10\kappa}{3}\underbrace{\frac{1}{NpT}\sum_{t=0}^{T-1}\sum_{n=1}^N\sum_{i=0}^{p-1} \mathbb{E}[ \| v_{t} - v_{t,i}^{(n)}\| ]}_{\text{consensus error for y}}   \ .
    \end{align}
\end{lemma}

From the above Lemma, two key error terms remain to be bounded: (i) the gradient error: $\mathbb{E}[ \| \nabla_x f(x_{t}, y_{t}) - u_{t}\| ] $ and $\mathbb{E}[ \| \nabla_y f(x_{t}, y_{t}) - v_{t}\| ]$, and (ii) the consensus error on momentum, $\mathbb{E}[ \| u_{t} - u_{t,i}^{(n)}\| ]$ and $\mathbb{E}[ \| v_{t} - v_{t,i}^{(n)}\| ]$. These error terms are bounded in Lemma~\ref{lemma:global_grad_error_x_y} and Lemma~\ref{lemma:local_grad_error_x_y}. 
In particular, Lemma~\ref{lemma:global_grad_error_x_y} has the terms:
\begin{align}
    \frac{1}{\beta_{x} T}\frac{2\sqrt{2}\sigma }{(Np)^{1-1/s}} \ , \quad \frac{2\sqrt{2}\beta_{x}^{1-1/s}}{(Np)^{1-1/s}}\sigma,
\end{align}
which explicitly demonstrate how heavy-tailed noises affects the convergence rate.

At last, combing the above error terms complete the convergence rate in Theorem~\ref{theorem:convergence-rate} and the comprehensive proof is provided in Appendix~\ref{app:norm}

\begin{figure*}[ht]
\begin{center}
\centerline{\includegraphics[scale=0.75]{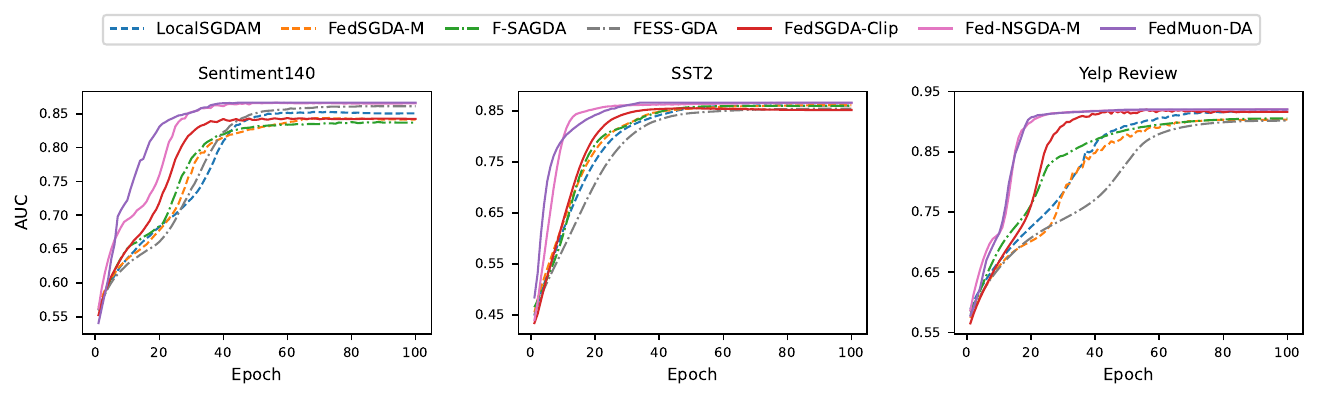}}
\vspace{-0.2in}
\caption{Testing AUC curves over epochs, $p = 4$, imbalance ratio $r=0.1$, i.i.d scenario.}
\label{fig:4}
\end{center}
\vspace{-0.3in}
\end{figure*}

\begin{figure*}[ht]
\begin{center}
\centerline{\includegraphics[scale=0.75]{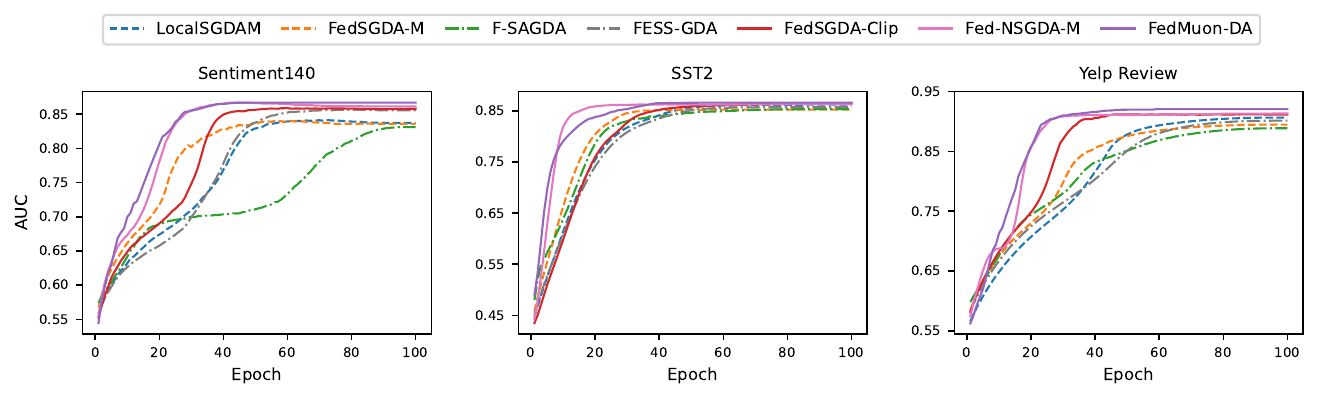}}
\vspace{-0.2in}
\caption{Testing AUC curves over epochs, $p = 16$, imbalance ratio $r=0.1$, i.i.d scenario.}
\label{fig:16}
\end{center}
\vspace{-0.3in}
\end{figure*}

\subsection{Convergence Rate of Algorithm~\ref{alg:fedmuon_update}}

In the following, we establish the convergence rate of Algorithm~\ref{alg:fedmuon_update}. 
\begin{theorem}\label{theorem:convergence-rate-muon}
	Given Assumptions~\ref{assumption:smooth}-\ref{assumption:ht_variance},
    by setting
    \begin{align}
        & \gamma_{x} = O\left(\frac{(Np)^{1/4}}{\kappa T^{3/4}}\right) \ , \quad \gamma_y = O\left(\kappa\gamma_{x}\right) \ ,  \notag \\
        & \beta_{x} = O\left(\frac{(Np)^{1/2}}{T^{1/2}}\right) \ ,  \quad \beta_y = O\left(\beta_x\right) \ ,  \\
        & \eta_x = O\left(\frac{1}{p\sqrt{T}}\right) \ , \quad \eta_y =   O\left(\eta_{x}\right) \ , 
        \notag 
    \end{align}
    we obtain
    \begin{align}
        \frac{1}{T}\sum_{t=0}^{T-1}\mathbb{E}[\| \nabla \Phi({X}_{t})  \|] & \leq O\left(\frac{\kappa}{(TNp)^{1/4}} + \frac{\kappa \sigma}{(TNp)^{\frac{s-1}{2s}}}\right) \ . 
    \end{align}
\end{theorem}

\begin{remark}
    By comparing Theorem~\ref{theorem:convergence-rate-muon} with Theorem~\ref{theorem:convergence-rate}, it is easy to know that FedMuon-DA has a convergence rate of $O\left(\frac{1}{(TNp)^{\frac{s-1}{2s}}}\right)$, a communication complexity of $O\left(\left(\frac{\kappa}{\epsilon}\right)^{\frac{3s}{2(s-1)}}\right)$, and  a learning rate ratio of $\gamma_x/\gamma_y=O(1/\kappa^2)$. 
\end{remark}

The proof of Theorem~\ref{theorem:convergence-rate-muon} follows the same outline as that of Theorem~\ref{theorem:convergence-rate}, with the complete details presented in Appendix~\ref{app:muon}. In what follows, we focus on explaining why Muon remains effective without clipping under heavy-tailed noise.
\begin{lemma}
    Given Assumptions~\ref{assumption:smooth}-\ref{assumption:ht_variance}, the following inequalities hold:
    \begin{align}
        &  \frac{1}{Np}\sum_{n=1}^N\sum_{i=0}^{p-1} \| X_{t,i}^{(n)} - X_{t}\|_F \leq \eta_{x}p\sqrt{n_x} \ , \notag \\
        &  \frac{1}{Np}\sum_{n=1}^N\sum_{i=0}^{p-1} \| Y_{t,i}^{(n)} - Y_{t}\|_F \leq \eta_{y}p\sqrt{n_y} \ .
    \end{align}
\end{lemma}
This lemma highlights the stabilizing effect of Muon under heavy-tailed noises. In the presence of heavy-tailed stochastic gradients, local updates $X_{t,i}^{(n)}$ may deviate significantly from the global parameter $X_{t}$, since the second moment of stochastic gradient variance can be unbounded. Traditional clipping addresses this issue by introducing a threshold hyperparameter $\tau$, while gradient normalization restricts the update to be independent of the gradient norm. Similarly, in Muon, the orthonormalization operation ensures that $\| P^{(n)}_{t,i}(Q^{(n)}_{t,i})^T \|_F\leq \sqrt{n_x}$, thereby restricting the deviation by the square root of the matrix dimension and ensuring robustness without explicit clipping.

\section{Experiments}
We conduct extensive experiments on imbalanced text classification tasks for deep AUC maximization under both homogeneous and heterogeneous settings. Specifically, we focus on text classification because language data is intrinsically \textbf{heavy-tailed}: word frequencies typically follow a power-law distribution (Zipf’s law)~\citep{piantadosi2014zipf, kunstner2024heavy}. As a traditional federated minimax framework, deep AUC maximization directly addresses the positive–negative imbalance in the text classification task. Moreover, data heterogeneity in federated learning further amplifies the heavy-tailed phenomenon~\citep{charles2021large, yang2022taming}, making both homogeneous (i.i.d.) and heterogeneous settings (non-i.i.d.) realistic and important for evaluation.

\paragraph{Deep AUC Maximization.} AUC (Area Under the ROC Curve) ~\citep{hanley1983method, elkan2001foundations} is a widely used metric for evaluating binary classification models, particularly valuable for imbalanced data, as it measures the ability to distinguish between positive and negative classes. Deep AUC maximization can be reformulated as a minimax problem~\citep{liu2020stochastic}, and we study the following federated formulation:
\begin{align}
    \min_{{w} \in \mathbb{R}^d,w_1, w_2} \max_{w_3} \frac{1}{N}\sum_{n=1}^N \mathbb{E} [f^{(n)}(w, w_1, w_2, w_3; \xi^{(n)})]  \ , \notag 
\end{align}

where $f^{(n)}$ is the AUC loss function on the $n$-th client:
	\begin{align}
		&  f(w, w_1, w_2, w_3; a, b)  \triangleq (1-p)(h(w;a)-w_1)^2\mathbb{I}_{[b=1]} \notag  \\ 
		&  + p(h(w;a)-w_2)^2\mathbb{I}_{[b=-1]}+ 2(1+w_3)(ph(w;a)\mathbb{I}_{[b=-1]} \notag \\
        & -(1-p)h(w;a)\mathbb{I}_{[b=1]}) -p(1-p)w_3^2 \ ,
	\end{align}
where $w \in \mathbb{R}^d$ denotes the model parameters, $(a,b)$ corresponds to a data sample with label, $h(w;a)$ is the prediction function implemented by the neural network, and $\mathbb{I}$ is the indicator function. The scalars $w_1, w_2, w_3$ serve as parameters in the AUC loss, and $p$ indicates the ratio of positive samples in the data distribution. By defining the primal variable as $(w^T, w_1, w_2)^T$ and the dual variable as $w_3$, the above problem can be expressed as a federated non-convex-PL problem. 

\begin{figure*}[ht]
\begin{center}
\centerline{\includegraphics[scale=0.75]{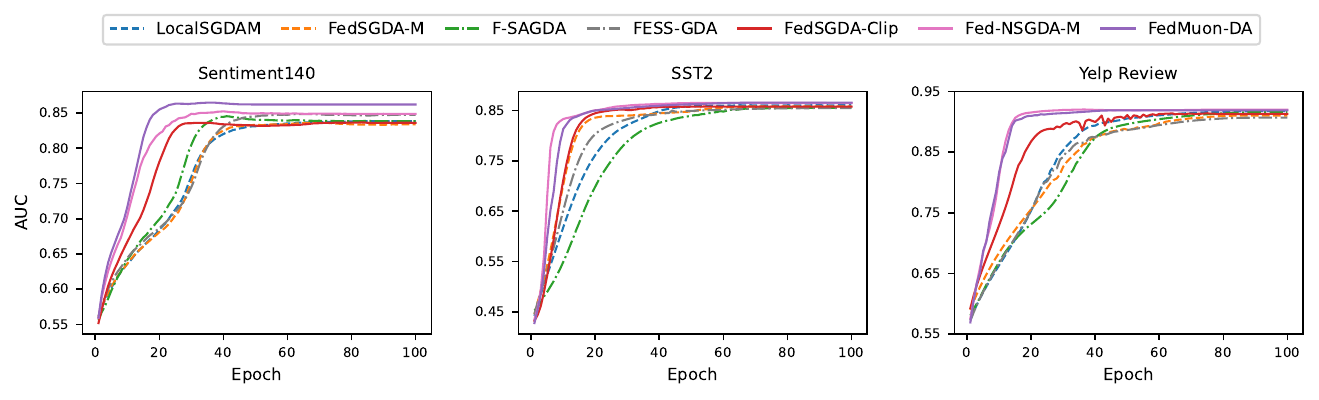}}
\vspace{-0.2in}
\caption{Testing AUC curves over epochs, $p = 4$, non-i.i.d scenario.}
\label{fig:heter}
\end{center}
\vspace{-0.3in}
\end{figure*}

\paragraph{Empirical Settings.} 
Our experiments are conducted on eight clients, with two clients allocated to each NVIDIA RTX 6000 GPU.
We evaluate our two methods on three widely used text classification benchmarks: Sentiment140~\citep{go2009twitter}, SST2~\citep{socher2013recursive}, and Yelp Review~\citep{zhang2015character},  all used in the binary classification setting (positive vs. negative). Sentiment140 is a large-scale Twitter sentiment dataset with automatically annotated tweets, SST2 is a benchmark dataset with phrase-level sentiment annotations, and Yelp Review consists of  user reviews with sentiment labels. To simulate imbalanced data, we construct both homogeneous and heterogeneous settings. In the homogeneous case, we randomly drop positive samples across all clients so that the ratio of positives $r$ is set to $0.1$ of the total. In the heterogeneous case, we adopt a more realistic setting where each client is assigned a distinct imbalance ratio. Specifically, the positive ratios for the eight clients are set to $[0.05, 0.05, 0.08, 0.1, 0.12, 0.15, 0.2, 0.25]$. Across all experiments, the batch size of each clients is 64. For the predictive model, we employ a two-layer recurrent neural network with input dimension $300$, hidden dimension $4096$, and output dimension $2$.

We compare our two methods with four state-of-the-art federated minimax algorithms: LocalSGDAM~\citep{sharma2022federated}, FedSGDA-M~\citep{wu2023solving}, F-SAGDA~\citep{yang2022sagda}, and FESS-GDA~\citep{shen2024stochastic}. To further highlight the benefits of normalized and orthonormalized updates over clipping, we also include a baseline, FedSGDA-Clip, obtained by replacing Step 6 in Algorithm~\ref{alg} with the clipping update:
\vspace{-5pt}
\begin{align}
    \quad x^{(n)}_{t,i+1} & = x^{(n)}_{t,i} - \eta_x \hat{u}^{(n)}_{t,i} \ , \notag  \\
    \text{where} \quad  \hat{u}^{(n)}_{t,i} & \triangleq \min\{1, \frac{\tau}{\|u^{(n)}_{t,i}\|}\} u^{(n)}_{t,i} \ ,
\end{align}
and similarly for $y^{(n)}_{t,i}$. In particular, we tune the learning rate of all baseline methods to achieve their best performance. For Fed-NSGDA-M and FedMuon-DA, the learning rate is selected from $[10^{-3}, 10^{-1}]$. The momentum parameter $\beta$ is fixed at $0.9$ for all baselines. For FedSGDA-Clip, we further tune the clipping threshold $\tau$ and fix it at $0.1$.

\vspace{-5pt}
\paragraph{Empirical Results.}
For the homogeneous setting (i.i.d scenario), we conduct experiments with communication period $p=4$ and $p=8$, and an imbalance ratio of $r=0.1$. The results are presented in Figure~\ref{fig:4} and \ref{fig:16}. Our two algorithms, Fed-NSGDA-M and FedMuon-DA, effectively address heavy-tailed noise in text data and consistently outperform the baselines in both convergence behavior and test performance. In particular, FedSGDA-Clip with a fixed clipping threshold $\tau$ exhibits varying performance across datasets, underscoring the necessity of hyperparameter tuning for different tasks, thereby limiting its practical applicability. Notably, such tuning is not required for our two methods, which remain stable and effective across different datasets and settings. 

For the heterogeneous setting (non-i.i.d scenario), we conduct experiments with communication period $p=4$, where each client is assigned a distinct imbalance ratio. The results are shown in Figure~\ref{fig:heter}. Since data heterogeneity amplifies heavy-tailed noise, this setting is more challenging, yet our two algorithms still outperform the baselines, further demonstrating their effectiveness and their potential for practical deployment in real-world federated applications.

\vspace{-5pt}
\section{Conclusion}
In this paper, we studied federated stochastic minimax optimization under heavy-tailed noise, a setting that better reflects modern large-scale models. We proposed two algorithms, Fed-NSGDA-M, which integrates normalized gradients, and FedMuon-DA, which leverages the Muon optimizer, to effectively address both heavy-tailed noise and data heterogeneity without requiring heterogeneity bounds. Our theoretical analysis provides the first rigorous and comprehensive guarantees for this setting, established that both methods achieve the same convergence rate of $O({1}/{(TNp)^{\frac{s-1}{2s}}})$, thereby contributing new insights into the design of federated minimax algorithms under heavy-tailed noise. Extensive experiments on imbalanced text classification tasks further demonstrated that our methods consistently outperform state-of-the-art baselines.

\newpage
{
	\bibliographystyle{abbrvnat}
	\bibliography{ref}
}

\newpage
\appendix
\onecolumn
\section{Appendix: Normalized Update}\label{app:norm}

\begin{lemma}\label{lemma:zijian-liu-lemma}
    \cite{liu2025nonconvex} Given a sequence of integrable random vectors $v_t\in \mathbb{R}^d$,  $\forall t\in \mathbb{N}$ such that $\mathbb{E}[v_t|\mathcal{F}_{t-1}] = 0$ where $\mathcal{F}_{t-1}$ is the natural filtration , then for any $s\in[1, 2]$, there is:
    \begin{align}
        & \mathbb{E}[\|\sum_{t=1}^{T}v_t\|]\leq 2\sqrt{2} \mathbb{E}[(\sum_{t=1}^{T}\|v_t\|^s)^{\frac{1}{s}}] \ , \quad T\in \mathbb{N} \ .
    \end{align}
\end{lemma}

\begin{lemma}\label{lemma:consensus_x_y}
    Given Assumptions~\ref{assumption:smooth}-\ref{assumption:ht_variance}, the following inequalities hold:
    \begin{align}
        & \frac{1}{Np}\sum_{n=1}^N\sum_{i=0}^{p-1} \| x_{t,i}^{(n)} - x_{t}\| \leq \eta_{x}p \ , \quad \frac{1}{Np}\sum_{n=1}^N\sum_{i=0}^{p-1} \| y_{t,i}^{(n)} - y_{t}\| \leq \eta_{y}p \ .
    \end{align}
    
\end{lemma}

\begin{proof}
    \begin{align}
		& \quad \|x_{t,i}^{(n)} - x_{t} \| \leq \sum_{j=0}^{i-1} \|x_{t,j+1}^{(n)} -  x_{t,j}^{(n)} \| \leq \eta_{x}\sum_{j=0}^{i-1}\Big\|\frac{u^{(n)}_{t,j}}{\|u^{(n)}_{t,j}\|} \Big\| \leq \eta_{x}p  \ , 
	\end{align}
    Taking the average over all $n$ and $i$ completes the proof. The argument for $y$ is identical.
\end{proof}

\begin{lemma}\label{lemma:phi_smooth}
    Given Assumptions~\ref{assumption:smooth}-\ref{assumption:ht_variance}, the following inequality holds:
    \begin{align}
        \mathbb{E}[\Phi(x_{t+1})] - \mathbb{E}[\Phi(x_{t})] &  \leq - \gamma_{x}\mathbb{E}[\|\nabla \Phi(x_{t})\|] + 2\gamma_{x}\kappa\mathbb{E}[\|\nabla_{y} f(x_{t},y_{t})  \|]  + 2\gamma_{x}\mathbb{E}[\|\nabla_{x} f(x_{t}, y_{t}) - u_{t} \|] \notag \\
        & \quad  + \frac{\gamma_{x}}{Np}\sum_{n=1}^N\sum_{i=0}^{p-1} \mathbb{E}[ \| u_{t} - u_{t,i}^{(n)}\| ] + \frac{L_{\Phi}\gamma_{x}^2}{2} \ . 
    \end{align}
\end{lemma}

\begin{proof}
    Due to the $L_{\Phi}$-smoothness of $\Phi(\cdot)$, we have
    \begin{align}\label{eq:l_phi}
    & \quad \mathbb{E}[\Phi(x_{t+1})]    \leq \mathbb{E}[\Phi(x_{t})] + \mathbb{E}[\langle \nabla \Phi(x_{t}), x_{t+1} - x_{t} \rangle] + \frac{L_{\Phi}}{2}\mathbb{E}[\| x_{t+1} - x_{t} \|^2] \notag \\
        & \overset{\scriptstyle (a)}{\leq}  \mathbb{E}[\Phi(x_{t})] - \gamma_{x} \mathbb{E}[\langle \nabla \Phi(x_{t}), \frac{1}{Np}\sum_{n=1}^N\sum_{i=0}^{p-1} \frac{u_{t,i}^{(n)}}{\|u_{t,i}^{(n)} \| }  \rangle] + \frac{L_{\Phi}\gamma_{x}^2}{2} \notag \\
        & = \mathbb{E}[\Phi(x_{t})] -\gamma_{x} \mathbb{E}[\langle \nabla \Phi(x_{t}) - u_{t}, \frac{1}{Np}\sum_{n=1}^N\sum_{i=0}^{p-1} \frac{u_{t,i}^{(n)}}{\|u_{t,i}^{(n)} \| }  \rangle] - \gamma_{x}\mathbb{E}[\langle u_{t}, \frac{1}{Np}\sum_{n=1}^N\sum_{i=0}^{p-1} \frac{u_{t,i}^{(n)}}{\|u_{t,i}^{(n)} \| }  \rangle] +\frac{L_{\Phi}\gamma_{x}^2}{2} \notag \\
        & \leq \mathbb{E}[\Phi(x_{t})] + \gamma_{x}\mathbb{E}[\|\nabla \Phi(x_{t}) - u_{t} \|\|\frac{1}{Np}\sum_{n=1}^N\sum_{i=0}^{p-1} \frac{u_{t,i}^{(n)}}{\|u_{t,i}^{(n)} \| }\|] - \gamma_{x} \mathbb{E}[\langle u_{t}, \frac{1}{Np}\sum_{n=1}^N\sum_{i=0}^{p-1} \frac{u_{t,i}^{(n)}}{\|u_{t,i}^{(n)} \| } - \frac{u_{t}}{\|u_{t}\| } \rangle] \notag \\
        & \quad - \gamma_{x}\mathbb{E}[\| u_{t}\|] + \frac{L_{\Phi}\gamma_{x}^2}{2} \notag \\
        & \leq \mathbb{E}[\Phi(x_{t})] + \gamma_{x}\mathbb{E}[\|\nabla \Phi(x_{t}) - u_{t} \|] - \gamma_{x}\mathbb{E}[\| u_{t}\|]  + \gamma_{x} \mathbb{E}[\| u_{t}\|\| \frac{1}{Np}\sum_{n=1}^N\sum_{i=0}^{p-1} \frac{u_{t,i}^{(n)}}{\|u_{t,i}^{(n)} \| } - \frac{u_{t}}{\|u_{t}\|\| } \|] + \frac{L_{\Phi}\gamma_{x}^2}{2} \notag \\
        & \overset{\scriptstyle (b)}{\leq} \mathbb{E}[\Phi(x_{t})]  + \gamma_{x}\mathbb{E}[\|\nabla \Phi(x_{t}) - u_{t} \|] - \gamma_{x}\mathbb{E}[\| u_{t}\|]  + \frac{\gamma_{x}}{Np}\sum_{n=1}^N\sum_{i=0}^{p-1} \mathbb{E}[ \| u_{t} - u_{t,i}^{(n)}\| ] + \frac{L_{\Phi}\gamma_{x}^2}{2} \notag \\
        & \overset{\scriptstyle (c)}{\leq} \mathbb{E}[\Phi(x_{t})]  - \gamma_{x}\mathbb{E}[\|\nabla \Phi(x_{t})\|] + 2\gamma_{x}\mathbb{E}[\|\nabla \Phi(x_{t}) - u_{t} \|]  + \frac{\gamma_{x}}{Np}\sum_{n=1}^N\sum_{i=0}^{p-1} \mathbb{E}[ \| u_{t} - u_{t,i}^{(n)}\| ] + \frac{L_{\Phi}\gamma_{x}^2}{2} \notag \\
        & \overset{\scriptstyle (d)}{\leq} \mathbb{E}[\Phi(x_{t})] - \gamma_{x}\mathbb{E}[\|\nabla \Phi(x_{t})\|] + 2\gamma_{x}\kappa\mathbb{E}[\|\nabla_{y} f(x_{t},y_{t})  \|]  + 2\gamma_{x}\mathbb{E}[\|\nabla_{x} f(x_{t}, y_{t}) - u_{t} \|] \notag \\
        & \quad  + \frac{\gamma_{x}}{Np}\sum_{n=1}^N\sum_{i=0}^{p-1} \mathbb{E}[ \| u_{t} - u_{t,i}^{(n)}\| ] + \frac{L_{\Phi}\gamma_{x}^2}{2} \ ,
    \end{align}
    where $(a)$ follows from $\|x_{t+1} - x_{t}\| = \| \frac{\gamma_{x}}{Np}\sum_{n=1}^N\sum_{i=0}^{p-1} \frac{u_{t,i}^{(n)}}{\|u_{t,i}^{(n)} \| } \| \leq \gamma_{x}$, $(b)$ can be bounded as follows:
    \begin{align}
        & \quad \|u_{t}\|\| \frac{1}{Np}\sum_{n=1}^N\sum_{i=0}^{p-1} \frac{u_{t,i}^{(n)}}{\|u_{t,i}^{(n)} \| } - \frac{u_{t}}{\|u_{t}\| }\| = \|u_{t}\| \Big\| \frac{1}{Np}\sum_{n=1}^N\sum_{i=0}^{p-1} \Big(\frac{u_{t,i}^{(n)}}{\|u_{t,i}^{(n)} \| } - \frac{u_{t,i}^{(n)}}{\|u_{t}\| } \Big) \Big\|  \notag \\
        & \leq  \frac{\|u_{t}\|}{Np}\sum_{n=1}^N\sum_{i=0}^{p-1} \Big( \Big\|  \frac{\|u_{t}\| - \|u_{t,i}^{(n)}\|}{\|u_{t,i}^{(n)} \|\|u_{t}\| }\Big\|\| u_{t,i}^{(n)}\| \Big) = \frac{1}{Np}\sum_{n=1}^N\sum_{i=0}^{p-1}  \Big\| \|u_{t}\| - \|u_{t,i}^{(n)}\| \Big\| \leq  \frac{1}{Np}\sum_{n=1}^N\sum_{i=0}^{p-1}  \| u_{t} - u_{t,i}^{(n)}\| 
    \end{align}
    where the first equality holds due to $u_{t} = \sum_{n=1}^N\sum_{i=0}^{p-1} u_{t,i}^{(n)}$, $(c)$ follows from $\|\nabla \Phi(x_{t})\| \leq \|\nabla \Phi(x_{t}) - u_{t}\| + \|u_{t}\|$,  and $(d)$ follows from
    \begin{align}\label{eq:phi_u}
        & \quad \mathbb{E}[\|\nabla \Phi(x_{t}) - u_{t} \|]  \leq \mathbb{E}[\|\nabla \Phi(x_{t}) - \nabla_{x} f(x_{t}, y_{t}) \|] + \mathbb{E}[\|\nabla_{x} f(x_{t}, y_{t}) - u_{t} \|] \notag \\
        & \leq L_f\mathbb{E}[\|y^*(x_{t}) - y_{t}\|] + \mathbb{E}[\|\nabla_{x} f(x_{t}, y_{t}) - u_{t} \|] \leq \kappa\mathbb{E}[\|\nabla_{y} f(x_{t}, y_{t}) \|] + \mathbb{E}[\|\nabla_{x} f(x_{t}, y_{t}) - u_{t} \|]
    \end{align}
    where the last step holds due to the inequality $\|y^*(x) - y\| \leq \frac{1}{\mu}\|\nabla_y f(x,y)\|$, as established in Appendix A of \cite{karimi2016linear}, and $\kappa = L_f/\mu$.
\end{proof}

\begin{lemma}\label{lemma:phi_f}
    Given Assumptions~\ref{assumption:smooth}-\ref{assumption:ht_variance}, the following inequality holds:
    \begin{align}
        & \quad  \mathbb{E}[\Phi(x_{t+1})] - \mathbb{E}[f(x_{t+1}, y_{t+1})]- ( \mathbb{E}[\Phi(x_{t})] -  \mathbb{E}[f(x_{t}, y_{t})] ) \notag \\
        &  \leq   (3\gamma_{x}\kappa - \gamma_{y}) \mathbb{E}[\|\nabla_{y} f(x_{t}, y_{t}) \|]  + 4\gamma_{x}\mathbb{E}[\|\nabla_{x} f(x_{t}, y_{t}) - u_{t} \|]  +  2\gamma_{y}\mathbb{E}[\| \nabla_y f(x_{t}, y_{t}) - v_{t}\|]   \notag \\
        & \quad  + \frac{2\gamma_{x}}{Np}\sum_{n=1}^N\sum_{i=0}^{p-1} \mathbb{E}[ \| u_{t} - u_{t,i}^{(n)}\| ]  + \frac{\gamma_{y}}{Np}\sum_{n=1}^N\sum_{i=0}^{p-1} \mathbb{E}[ \| v_{t} - v_{t,i}^{(n)}\| ]  + \frac{L_{\Phi}\gamma_{x}^2}{2} + \frac{L_f(\gamma_{x}+\gamma_{y})^2}{2}  \ . 
    \end{align}
\end{lemma}

\begin{proof}
Following Eq.~(\ref{eq:l_phi}), due to the smoothness of $f$ regarding $y$, we obtain
\begin{align}
    & \quad \mathbb{E}[f(x_{t+1}, y_{t})] \leq \mathbb{E}[f(x_{t+1}, y_{t+1})] - \mathbb{E}[\langle \nabla_y f(x_{t+1}, y_{t}) , y_{t+1} - y_{t} \rangle] + \frac{L_f}{2}\mathbb{E}[\|y_{t+1} - y_{t}\|^2] \notag \\
    & \overset{\scriptstyle (a)}{\leq}  \mathbb{E}[f(x_{t+1}, y_{t+1})] - \gamma_{y} \mathbb{E}[\langle \nabla_{y} f(x_{t+1}, y_{t}), \frac{1}{Np}\sum_{n=1}^N\sum_{i=0}^{p-1} \frac{v_{t,i}^{(n)}}{\|v_{t,i}^{(n)} \| }  \rangle] + \frac{L_{f}\gamma_{y}^2}{2} \notag \\
    & \leq \mathbb{E}[f(x_{t+1}, y_{t+1})]  + \gamma_{y}\mathbb{E}[\| \nabla_y f(x_{t+1}, y_{t}) - v_{t}\|] - \gamma_{y}\mathbb{E}[\| v_{t}\|]  + \frac{\gamma_{y}}{Np}\sum_{n=1}^N\sum_{i=0}^{p-1} \mathbb{E}[ \| v_{t} - v_{t,i}^{(n)}\| ]  + \frac{L_f\gamma_y^2}{2} \notag \\
    & \leq \mathbb{E}[f(x_{t+1}, y_{t+1})]  + \gamma_{y}\mathbb{E}[\| \nabla_y f(x_{t+1}, y_{t}) - \nabla_y f(x_{t}, y_{t}) \|] +  \gamma_{y}\mathbb{E}[\| \nabla_y f(x_{t}, y_{t}) - v_{t}\|] - \gamma_{y}\mathbb{E}[\| v_{t}\|] \notag \\
    & \quad + \frac{\gamma_{y}}{Np}\sum_{n=1}^N\sum_{i=0}^{p-1} \mathbb{E}[ \| v_{t} - v_{t,i}^{(n)}\| ]  + \frac{L_f\gamma_y^2}{2} \notag \\
    & \overset{\scriptstyle (b)}{\leq} \mathbb{E}[f(x_{t+1}, y_{t+1})]  + \gamma_{y}L_f\mathbb{E}[\| x_{t+1} - x_{t} \|] +  2\gamma_{y}\mathbb{E}[\| \nabla_y f(x_{t}, y_{t}) - v_{t}\|] - \gamma_{y}\mathbb{E}[\| \nabla_y f(x_{t}, y_{t})\|] \notag \\
    & \quad + \frac{\gamma_{y}}{Np}\sum_{n=1}^N\sum_{i=0}^{p-1} \mathbb{E}[ \| v_{t} - v_{t,i}^{(n)}\| ]  + \frac{L_f\gamma_y^2}{2} \notag \\
    & \leq \mathbb{E}[f(x_{t+1}, y_{t+1})]  +  2\gamma_{y}\mathbb{E}[\| \nabla_y f(x_{t}, y_{t}) - v_{t}\|] - \gamma_{y}\mathbb{E}[\| \nabla_y f(x_{t}, y_{t})\|] + \frac{\gamma_{y}}{Np}\sum_{n=1}^N\sum_{i=0}^{p-1} \mathbb{E}[ \| v_{t} - v_{t,i}^{(n)}\| ] \notag \\
    & \quad + \frac{L_f\gamma_y(\gamma_{y}+2\gamma_{x})}{2}  \ , 
\end{align}
$(a)$ follows from $\|y_{t+1} - y_{t}\| = \| \frac{\gamma_{y}}{Np}\sum_{n=1}^N\sum_{i=0}^{p-1} \frac{v_{t,i}^{(n)}}{\|v_{t,i}^{(n)} \| } \| \leq \gamma_{y}$, $(b)$ follows from $\|\nabla_y f(x_{t}, y_{t})\| \leq \|\nabla_y f(x_{t}, y_{t}) - v_{t}\| + \|v_{t}\|$.

Similarly, due to the smoothness of $f$ regarding $x$, we obtain
\begin{align}
    & \quad \mathbb{E}[f(x_{t}, y_{t})] \leq \mathbb{E}[f(x_{t+1},y_{t})] - \mathbb{E}[\langle \nabla_{x} f(x_{t}, y_{t}), x_{t+1} - x_{t} \rangle] + \frac{L_f}{2}\mathbb{E}[\|x_{t+1} - x_{t}\|^2] \notag \\
    & \leq \mathbb{E}[f(x_{t+1},y_{t})] + \gamma_{x}\mathbb{E}[\langle \nabla_{x} f(x_{t}, y_{t}), \frac{1}{Np}\sum_{n=1}^N\sum_{i=0}^{p-1} \frac{u_{t,i}^{(n)}}{\|u_{t,i}^{(n)} \| }  \rangle] + \frac{L_f\gamma_{x}^2}{2} \notag \\
    & \leq \mathbb{E}[f(x_{t+1},y_{t})] + \gamma_{x}\mathbb{E}[\langle \nabla_{x} f(x_{t}, y_{t}) - u_{t}, \frac{1}{Np}\sum_{n=1}^N\sum_{i=0}^{p-1} \frac{u_{t,i}^{(n)}}{\|u_{t,i}^{(n)} \| }  \rangle]  + \gamma_{x}\mathbb{E}[\langle u_{t}, \frac{1}{Np}\sum_{n=1}^N\sum_{i=0}^{p-1} \frac{u_{t,i}^{(n)}}{\|u_{t,i}^{(n)} \| }  \rangle] + \frac{L_f\gamma_{x}^2}{2} \notag \\
    & \leq \mathbb{E}[f(x_{t+1},y_{t})] + \gamma_{x}\mathbb{E}[\langle \nabla_{x} f(x_{t}, y_{t}) - u_{t}, \frac{1}{Np}\sum_{n=1}^N\sum_{i=0}^{p-1} \frac{u_{t,i}^{(n)}}{\|u_{t,i}^{(n)} \| }  \rangle]  + \gamma_{x}\mathbb{E}[\langle u_{t}, \frac{1}{Np}\sum_{n=1}^N\sum_{i=0}^{p-1} \frac{u_{t,i}^{(n)}}{\|u_{t,i}^{(n)} \| } - \frac{u_{t}}{\|u_{t}\| }  \rangle] \notag \\
    & \quad + \gamma_{x}\mathbb{E}[\|u_{t} \|]  + \frac{L_f\gamma_{x}^2}{2} \notag \\
    & \leq  \mathbb{E}[f(x_{t+1},y_{t})]  + \gamma_{x}\mathbb{E}[\|\nabla_{x} f(x_{t},y_{t}) - u_{t} \|] + \gamma_{x}\mathbb{E}[\|u_{t} \|]  + \frac{\gamma_{x}}{Np}\sum_{n=1}^N\sum_{i=0}^{p-1} \mathbb{E}[ \| u_{t} - u_{t,i}^{(n)}\| ] + \frac{L_f\gamma_{x}^2}{2} \notag \\
    & \overset{\scriptstyle (a)}{\leq} \mathbb{E}[f(x_{t+1},y_{t})]  + 2\gamma_{x}\mathbb{E}[\|\nabla_{x} f(x_{t},y_{t}) - u_{t} \|]   + \frac{\gamma_{x}}{Np}\sum_{n=1}^N\sum_{i=0}^{p-1} \mathbb{E}[ \| u_{t} - u_{t,i}^{(n)}\| ] +  \gamma_{x}\mathbb{E}[\|\nabla \Phi({x})  \|] \notag \\
    & \quad +  \gamma_{x}\kappa \mathbb{E}[\|\nabla_{y} f(x_{t}, y_{t}) \|] + \frac{L_f\gamma_{x}^2}{2} \ , 
\end{align}
where $(a)$ follows from $\mathbb{E}[\| {u}_{t}\|] \leq \mathbb{E}[\| {u}_{t} - \nabla \Phi({x})\|] + \mathbb{E}[\|\nabla \Phi({x}) \|]  \overset{\scriptstyle \text{Eq.~(\ref{eq:phi_u})}}{\leq}\kappa\mathbb{E}[\|\nabla_{y} f(x_{t}, y_{t}) \|] + \mathbb{E}[\|\nabla_{x} f(x_{t}, y_{t}) - u_{t} \|] + \mathbb{E}[\|\nabla \Phi({x}) \|] \ . $
By combining the above two inequalities, we obtain
\begin{align}
    & \mathbb{E}[f(x_{t}, y_{t})] - \mathbb{E}[f(x_{t+1}, y_{t+1})] \leq  \gamma_{x}\mathbb{E}[\|\nabla \Phi({x})  \|]  +  (\gamma_{x}\kappa - \gamma_{y}) \mathbb{E}[\|\nabla_{y} f(x_{t}, y_{t}) \|] \notag \\
    &  + 2\gamma_{x}\mathbb{E}[\|\nabla_{x} f(x_{t},y_{t}) - u_{t} \|]  +  2\gamma_{y}\mathbb{E}[\| \nabla_y f(x_{t}, y_{t}) - v_{t}\|]   \notag \\
    &  + \frac{\gamma_{x}}{Np}\sum_{n=1}^N\sum_{i=0}^{p-1} \mathbb{E}[ \| u_{t} - u_{t,i}^{(n)}\| ] + \frac{\gamma_{y}}{Np}\sum_{n=1}^N\sum_{i=0}^{p-1} \mathbb{E}[ \| v_{t} - v_{t,i}^{(n)}\| ] + \frac{L_f(\gamma_{x}+\gamma_{y})^2}{2}  \ .
\end{align}
The proof is complete by applying Lemma~\ref{lemma:phi_smooth}.
\end{proof}

\begin{lemma}\label{lemma:phi_sum}
    Given Assumptions~\ref{assumption:smooth}-\ref{assumption:ht_variance}, by setting $\gamma_{x} = \frac{\gamma_{y}}{10\kappa}$, the following inequality holds:
    \begin{align}\label{eq:phi_sum}
        & \frac{1}{T}\sum_{t=0}^{T-1}\mathbb{E}[\| \nabla \Phi({x}_{t})  \|] \leq \frac{(\Phi(x_{0})- \Phi^*) }{\gamma_{x}T} + \frac{\Phi(x_{0}) - f(x_0,y_0)}{3\gamma_{x}T} + \frac{2L_{\Phi}\gamma_{x}}{3}  + \frac{L_f\gamma_{x}(1+10\kappa)^2}{6} \notag \\
        & + \frac{10}{3}\frac{1}{T}\sum_{t=0}^{T-1}\mathbb{E}[\|\nabla_{x} f(x_{t}, y_{t}) - u_{t} \|] +  \frac{20\kappa}{3}\frac{1}{T}\sum_{t=0}^{T-1}\mathbb{E}[\| \nabla_y f(x_{t}, y_{t}) - v_{t}\|] \notag \\
        &  + \frac{5}{3NpT}\sum_{t=0}^{T-1}\sum_{n=1}^N\sum_{i=0}^{p-1} \mathbb{E}[ \| u_{t} - u_{t,i}^{(n)}\| ] + \frac{10\kappa}{3NpT}\sum_{t=0}^{T-1}\sum_{n=1}^N\sum_{i=0}^{p-1} \mathbb{E}[ \| v_{t} - v_{t,i}^{(n)}\| ]  \ .
    \end{align}
\end{lemma}

\begin{proof}
    From the potential function:
    \begin{align}
        \mathcal{L}_{t} = 3\mathbb{E}[\Phi(x_{t})]  + (\mathbb{E}[\Phi(x_{t})] -  \mathbb{E}[f(x_{t}, y_{t})] ) \ , 
    \end{align}
    and applying Lemma~\ref{lemma:phi_smooth} and Lemma~\ref{lemma:phi_f}, we obtain:
    \begin{align}
        & \mathcal{L}_{t+1} - \mathcal{L}_{t} \leq  - 3\gamma_{x}\mathbb{E}[\|\nabla \Phi(x_{t})\|] + (9\gamma_{x}\kappa - \gamma_{y}) \mathbb{E}[\|\nabla_{y} f(x_{t},y_{t})  \|]  + 10\gamma_{x}\mathbb{E}[\|\nabla_{x} f(x_{t}, y_{t}) - u_{t} \|] \notag \\
        & \quad  +  2\gamma_{y}\mathbb{E}[\| \nabla_y f(x_{t}, y_{t}) - v_{t}\|]  + \frac{5\gamma_{x}}{Np}\sum_{n=1}^N\sum_{i=0}^{p-1} \mathbb{E}[ \| u_{t} - u_{t,i}^{(n)}\| ] + \frac{\gamma_{y}}{Np}\sum_{n=1}^N\sum_{i=0}^{p-1} \mathbb{E}[ \| v_{t} - v_{t,i}^{(n)}\| ] \notag \\
        & \quad + 2L_{\Phi}\gamma_{x}^2  + \frac{L_f(\gamma_{x}+\gamma_{y})^2}{2} \ . 
    \end{align}
    With $\gamma_{x} = \frac{\gamma_{y}}{10\kappa}$, the coefficient of term $\mathbb{E}[\|\nabla_{y} f(x_{t},y_{t})  \|]$ is $9\gamma_{x}\kappa - \gamma_{y}= -\frac{1}{10}\gamma_{y}$. Hence, this negative term can be discarded.
    
    By summing the above inequality over $t$ and rearrange the terms, the proof is concluded.
\end{proof}

In the following, we establish two lemmas to further bound the remaining terms in Eq.~(\ref{eq:phi_sum}). After deriving an upper bound for each term, we complete the convergence rate analysis.

\begin{lemma}\label{lemma:global_grad_error_x_y}
    Given Assumptions~\ref{assumption:smooth}-\ref{assumption:ht_variance}, the gradient error regarding variable $x$ is bounded as:
    \begin{align}
        \frac{1}{T}\sum_{t=0}^{T-1}\mathbb{E}[ \| \nabla_x f(x_{t}, y_{t}) - u_{t}\| ] & \leq \frac{(\eta_{x}+\eta_{y})pL_f}{\beta_{x} T} +\frac{1}{\beta_{x} T}\frac{2\sqrt{2}\sigma }{(Np)^{1-1/s}}  + \frac{(\gamma_{x}+\gamma_{y})L_f}{\beta_{x}} \notag \\
        & \quad +  (\eta_{x}+\eta_{y}) pL_f + \frac{2\sqrt{2}\beta_{x}^{1-1/s}}{(Np)^{1-1/s}}\sigma \ ,
    \end{align}
    the gradient error regarding variable $y$ is bounded as:
    \begin{align}
        \frac{1}{T}\sum_{t=0}^{T-1}\mathbb{E}[ \| \nabla_y f(x_{t}, y_{t}) - v_{t}\| ] & \leq \frac{(\eta_{x}+\eta_{y})pL_f}{\beta_{y} T} +\frac{1}{\beta_{y} T}\frac{2\sqrt{2}\sigma }{(Np)^{1-1/s}}  + \frac{(\gamma_{x}+\gamma_{y})L_f}{\beta_{y}} \notag \\
        & \quad +  (\eta_{x}+\eta_{y}) pL_f + \frac{2\sqrt{2}\beta_{y}^{1-1/s}}{(Np)^{1-1/s}}\sigma \ .
    \end{align}
\end{lemma}

\begin{proof}
From the update rule of $u_{t}$, we have
\begin{align}\label{eq:global-grad-error-x}
    & \quad \mathbb{E}[ \| \nabla_{x} f(x_{t}, y_{t}) - u_{t}\| ]  = \mathbb{E}\Big[\Big \| (1 - \beta_{x})\Big(\nabla_{x} f(x_{t},y_{t}) - \nabla_{x} f(x_{t-1}, y_{t-1}) + \nabla_{x} f(x_{t-1}, y_{t-1}) - u_{t-1}\Big) \notag \\
    & \qquad \quad + \beta_{x} \Big( \nabla_{x} f(x_{t},y_{t}) - \frac{1}{Np}\sum_{n=1}^N \sum_{i=0}^{p-1} \nabla_{x} f^{(n)}(x_{t,i}^{(n)}, y_{t,i}^{(n)}; \xi_{t,i}^{(n)})  \Big)\Big\|\Big] \notag \\
    & \leq (1-\beta_{x})^{t} \mathbb{E}\Big[\Big\|\nabla_{x} f(x_{0}, y_{0}) - \frac{1}{Np}\sum_{n=1}^N\sum_{i=0}^{p-1}\nabla_{x} f^{(n)}(x_{0,i}^{(n)}, y_{0,i}^{(n)}; \xi_{0,i}^{(n)} ) \Big\|\Big] \notag \\
    & \quad +  \sum_{\tau = 1}^{t}(1-\beta)^{t-\tau+1}\mathbb{E}[\| \nabla_{x} f(x_{\tau}, y_{\tau}) - \nabla_{x} f(x_{\tau-1}, y_{\tau-1}) \|] \notag \\
    & \quad + \mathbb{E}\Big[\Big\|\beta\sum_{\tau = 1}^{t}(1-\beta)^{t-\tau} \Big( \nabla_{x} f(x_{\tau}, y_{\tau}) - \frac{1}{Np}\sum_{n=1}^N \sum_{i=0}^{p-1} \nabla_{x} f^{(n)}(x_{\tau,i}^{(n)}, y_{\tau,i}^{(n)}; \xi_{\tau,i}^{(n)})  \Big)  \Big\| \Big] \notag \\
    & \triangleq (1-\beta)^{t}T_1 + T_2 + T_3 \ .
\end{align}
To simplify the following proof, we define the gradient variance with local value as:
\begin{align}\label{eq:def_delta_local}
    {\delta}_{t,i}^{(n)} = \nabla_{x} f^{(n)}(x_{t,i}^{(n)}, y_{t,i}^{(n)}) - \nabla_{x} f^{(n)}(x_{t,i}^{(n)}, y_{t,i}^{(n)}; \xi_{t,i}^{(n)}) \ . 
\end{align}
By Assumption \ref{assumption:ht_variance}, it follows that $\mathbb{E}[{\delta}_{t,i}^{(n)}] = 0$,  $\mathbb{E}[\|{\delta}_{t,i}^{(n)}\|^s] \leq \sigma^s$.

Consider $T_1$ in Eq.~(\ref{eq:global-grad-error-x}),
\begin{align}\label{eq:global_t_1}
    & \quad T_1 = \mathbb{E}\Big[\Big\|\nabla_{x} f(x_{0}, y_{0}) - \frac{1}{Np}\sum_{n=1}^N\sum_{i=0}^{p-1}\nabla_{x} f^{(n)}(x_{0,i}^{(n)}, y_{0,i}^{(n)}; \xi_{0,i}^{(n)} ) \Big\|\Big] \notag \\
    & \leq \mathbb{E}\Big[\Big\| \frac{1}{N}\sum_{n=1}^N\nabla_{x} f^{(n)}(x_{0},y_{0}) - \frac{1}{Np}\sum_{n=1}^N\sum_{i=0}^{p-1}\nabla_{x} f^{(n)}(x_{0,i}^{(n)}, y_{0,i}^{(n)})\Big\|\Big] \notag \\
    & \quad +  \mathbb{E}\Big[\Big\|\frac{1}{Np}\sum_{n=1}^N\sum_{i=0}^{p-1} \Big(\nabla_{x} f^{(n)}(x_{0,i}^{(n)}, y_{0,i}^{(n)}) - \nabla_{x} f^{(n)}(x_{0,i}^{(n)}, y_{0,i}^{(n)}; \xi_{0,i}^{(n)} ) \Big) \Big\|\Big]  \notag \\
    &  \overset{\scriptstyle (a)}{\leq} \frac{ L_f}{Np}\sum_{n=1}^N\sum_{i=0}^{p-1}\mathbb{E}[\|x_{0} - x_{0,i}^{(n)} \|] + \frac{ L_f}{Np}\sum_{n=1}^N\sum_{i=0}^{p-1}\mathbb{E}[\|y_{0} - y_{0,i}^{(n)} \|] +  \mathbb{E}\Big[\Big\|\frac{1}{Np}\sum_{n=1}^N\sum_{i=0}^{p-1} {\delta}_{0,i}^{(n)} \Big\|\Big]  \notag \\
    & \overset{\scriptstyle \text{Lemma~\ref{lemma:consensus_x_y}}}{\leq} (\eta_{x}+\eta_{y})pL_f+ \frac{2\sqrt{2}\sigma}{(Np)^{1-1/s}} \ , 
\end{align}
$(a)$ follows from the definition of ${\delta}_{t,i}^{(n)}$ with $t=0$, and the last step is derived as follows:
\begin{align}
    & \quad  \mathbb{E}[\| \frac{1}{Np}\sum_{n=1}^N\sum_{i=0}^{p-1}{\delta}_{0,i}^{(n)} \|] \overset{\scriptstyle \text{Lemma~\ref{lemma:zijian-liu-lemma}}}{\leq} \frac{2\sqrt{2}}{Np} \mathbb{E}\Big[\Big(\sum_{n=1}^N\sum_{i=0}^{p-1}\| {\delta}_{0,i}^{(n)} \|^s\Big)^{\frac{1}{s}}\Big] \overset{\scriptstyle (a)}{\leq} \frac{2\sqrt{2}}{Np}\Big(\sum_{n=1}^N \sum_{i=0}^{p-1}\mathbb{E}[\| {\delta}_{0,i}^{(n)}  \|^s] \Big)^{\frac{1}{s}}\notag \\
    & \overset{\scriptstyle \text{Assumption~\ref{assumption:ht_variance}}}{\leq} \frac{2\sqrt{2}\sigma }{(Np)^{1-1/s}}\ ,
\end{align}
where $(a)$ holds due to Hölder’s inequality. 

Consider $T_2$ in Eq.~(\ref{eq:global-grad-error-x}),  we first derive the following bound:
\begin{align}\label{eq:grad-increment-x}
	& \quad \mathbb{E}[\| \nabla_{x} f(x_{\tau},y_{\tau})  - \nabla_{x} f(x_{\tau-1}, y_{\tau-1} )  \|]  \leq L_f\mathbb{E}[\|(x_{\tau},y_{\tau}) - (x_{\tau-1}, y_{\tau-1} )  \|] \notag \\
    & \leq  L_f\gamma_{x} \mathbb{E}[\|\frac{1}{ Np}\sum_{n=1}^N\sum_{i=0}^{p-1}\frac{u_{\tau,i}^{(n)}}{\|u_{\tau,i}^{(n)}\|} \|] + L_f\gamma_{y}\mathbb{E}[\|\frac{1}{ Np}\sum_{n=1}^N\sum_{i=0}^{p-1}\frac{v_{\tau,i}^{(n)}}{\|v_{\tau,i}^{(n)}\|} \|]  \notag \\
    & \leq  (\gamma_{x}+\gamma_{y})L_f  \ . 
\end{align}
Then, from $\sum_{\tau=1}^{t}(1-\beta_{x})^{(t-\tau+1)} \leq \frac{1-\beta_{x}}{\beta_{x}} \leq \frac{1}{\beta_{x}}$,  since $\beta_{x}<1$. As a result, we obtain: $ T_2 \leq \frac{(\gamma_{x}+\gamma_{y})L_f}{\beta_{x}}$ .

For $T_3$ in Eq.~(\ref{eq:global-grad-error-x}), from the definition of $\delta_{t,i}^{(n)}$, we bound it as follows:
\begin{align}
    & \quad T_3 = \mathbb{E}\Big[\Big\|\beta_{x}\sum_{\tau = 1}^{t}(1-\beta_{x})^{t-\tau} \Big( \nabla_{x} f(x_{\tau}, y_{\tau}) - \frac{1}{Np}\sum_{n=1}^N \sum_{i=0}^{p-1} \nabla_{x} f^{(n)}(x_{\tau,i}^{(n)}, y_{\tau,i}^{(n)}; \xi_{\tau,i}^{(n)})  \Big)  \Big\| \Big] \notag \\
    & \leq \beta_{x}\sum_{\tau = 1}^{t}(1-\beta_{x})^{t-\tau} \mathbb{E}\Big[\Big\|  \nabla_{x} f(x_{\tau}, y_{\tau}) - \frac{1}{Np}\sum_{n=1}^N \sum_{i=0}^{p-1} \nabla_{x} f^{(n)}(x_{\tau,i}^{(n)}, y_{\tau,i}^{(n)})   \Big\| \Big] + \mathbb{E}\Big[\Big\|\beta_{x}\sum_{\tau = 1}^{t}(1-\beta_{x})^{t-\tau} \frac{1}{Np}\sum_{n=1}^N \sum_{i=0}^{p-1} \delta_{\tau,i}^{(n)}  \Big\| \Big] \notag \\
    & \leq \frac{L_f}{Np}\sum_{n=1}^N \sum_{i=0}^{p-1} \mathbb{E}[\|  x_{\tau} -x_{\tau,i}^{(n)} \|] + \frac{L_f}{Np}\sum_{n=1}^N \sum_{i=0}^{p-1} \mathbb{E}[\|  y_{\tau} -y_{\tau,i}^{(n)} \|] + \mathbb{E}\Big[\Big\|\beta_{x}\sum_{\tau = 1}^{t}(1-\beta_{x})^{t-\tau} \frac{1}{Np}\sum_{n=1}^N \sum_{i=0}^{p-1} \delta_{\tau,i}^{(n)}  \Big\| \Big] \notag \\
    & \overset{\scriptstyle \text{Lemma~\ref{lemma:consensus_x_y}}}{\leq}  (\eta_{x}+\eta_{y})pL_f + \frac{2\sqrt{2}\beta_{x}^{1-1/s}}{(Np)^{1-1/s}}\sigma 
\end{align}
where the last step holds due to:
\begin{align}\label{eq:u_variance}
	& \quad  \mathbb{E}[\| \beta_{x}\sum_{\tau=1}^{t}(1-\beta_{x})^{t-\tau}\frac{1}{Np}\sum_{n=1}^N \sum_{i=0}^{p-1} \delta_{\tau,i}^{(n)} \|] = \frac{1}{Np}\mathbb{E}[\| \beta_{x}\sum_{\tau=1}^{t}(1-\beta_{x})^{t-\tau}\sum_{n=1}^N \sum_{i=0}^{p-1} \delta_{\tau,i}^{(n)} \|]  \notag \\
    & \overset{\scriptstyle \text{Lemma~\ref{lemma:zijian-liu-lemma}}}{\leq} \frac{2\sqrt{2}}{Np}\mathbb{E}\Big[ \Big(\sum_{\tau=1}^{t}\sum_{n=1}^N \sum_{i=0}^{p-1} \|\beta_{x}(1-\beta_{x})^{t-\tau}\delta_{\tau,i}^{(n)}\|^s\Big)^{1/s}\Big] \notag \\
	&  = \frac{2\sqrt{2}}{Np}\mathbb{E}\Big[ \Big(\sum_{\tau=1}^{t}\sum_{n=1}^N \sum_{i=0}^{p-1}\beta_{x}^{s}(1-\beta_{x})^{s(t-\tau)}\|\delta_{\tau,i}^{(n)} \|^s\Big)^{1/s}\Big]  \notag \\
    & \overset{\scriptstyle (a)}{\leq}  \frac{2\sqrt{2}}{Np}\Big(\mathbb{E} \Big[\sum_{\tau=1}^{t}\sum_{n=1}^N \sum_{i=0}^{p-1}\beta_{x}^{s}(1-\beta_{x})^{s(t-\tau)}\|\delta_{\tau,i}^{(n)} \|^s\Big]\Big)^{1/s}  \overset{\scriptstyle (b)}{\leq} \frac{2\sqrt{2}\beta_{x}^{1-1/s}}{(Np)^{1-1/s}}\sigma \ , 
\end{align}
where $(a)$ is due to Hölder’s inequality, $(b)$ follows from Assumption~\ref{assumption:ht_variance} and 
\begin{align}
    & \Big(\sum_{\tau=1}^{t}(1-\beta_{x})^{s(t-\tau)}\Big)^{1/s} \leq \Big(\frac{1}{1-(1-\beta_{x})^{s}}\Big)^{1/s} \leq  \Big(\frac{1}{1-(1-\beta_{x})}\Big)^{1/s} \leq \beta_{x}^{-1/s} \ .
\end{align}
Finally, by substituting $T_1$, $T_2$, and $T_3$ into Eq.~(\ref{eq:global-grad-error-x}), we obtain:
\begin{align}
     & \quad \mathbb{E}[ \| \nabla_{x} f(x_{t}, y_{t}) - u_{t}\| ] \notag \\
     & \leq (1-\beta_{x})^{t}\Big((\eta_{x}+\eta_{y}) pL_f+\frac{2\sqrt{2}\sigma }{(Np)^{1-1/s}}\Big) + \frac{(\gamma_{x}+\gamma_{y}) L_f}{\beta_{x}} +  (\eta_{x}+\eta_{y})pL_f + \frac{2\sqrt{2}\beta_{x}^{1-1/s}}{(Np)^{1-1/s}}\sigma \ .
\end{align}
Summing up from $t=0$ to $T-1$, we obtain
\begin{align}
    & \quad \frac{1}{T}\sum_{t=0}^{T-1}\mathbb{E}[ \| \nabla_{x} f(x_{t},y_{t}) - u_{t}\| ] \notag \\
    & \leq  \frac{1}{T}\sum_{t=0}^{T-1}(1-\beta_{x})^{t}\Big((\eta_{x}+\eta_{y})pL_f+\frac{2\sqrt{2}\sigma }{(Np)^{1-1/s}}\Big) + \frac{(\gamma_{x}+\gamma_{y})L_f}{\beta_{x}} + (\eta_{x}+\eta_{y})pL_f + \frac{2\sqrt{2}\beta_{x}^{1-1/s}}{(Np)^{1-1/s}}\sigma  \notag \\
    & \leq \frac{(\eta_{x}+\eta_{y})pL_f}{\beta_{x} T} +\frac{1}{\beta_{x} T}\frac{2\sqrt{2}\sigma }{(Np)^{1-1/s}} + \frac{(\gamma_{x}+\gamma_{y})L_f}{\beta_{x}} +  (\eta_{x}+\eta_{y}) pL_f + \frac{2\sqrt{2}\beta_{x}^{1-1/s}}{(Np)^{1-1/s}}\sigma  \ .
\end{align}
Similarly, the second inequality in the lemma can be proved by following the same line of  reasoning. Thus, the proof is complete.
\end{proof}

\begin{lemma}\label{lemma:local_grad_error_x_y}
    Given Assumptions~\ref{assumption:smooth}-\ref{assumption:ht_variance}, the consensus error on momentum regarding variable $x$ is bounded as:
    \begin{align}
        \frac{1}{NpT}\sum_{t=0}^{T-1}\sum_{n=1}^N\sum_{i=0}^{p-1} \mathbb{E}[ \| u_{t} - u_{t,i}^{(n)}\| ] & \leq 8\sqrt{2}\beta_{x}\sigma + 4\beta_{x}( \eta_{x}+\eta_{y}) pL_f +  2\beta_{x}(\gamma_{x}+\gamma_{y})L_f    \ ,
    \end{align}
    the consensus error on momentum regarding variable $y$ is bounded as:
    \begin{align}
        \frac{1}{NpT}\sum_{t=0}^{T-1}\sum_{n=1}^N\sum_{i=0}^{p-1} \mathbb{E}[ \| v_{t} - v_{t,i}^{(n)}\| ] & \leq 8\sqrt{2}\beta_{y}\sigma + 4\beta_{y}( \eta_{x}+\eta_{y}) pL_f  +  2\beta_{y}(\gamma_{x}+\gamma_{y})L_f    \ .
    \end{align}
\end{lemma}

\begin{proof}
Since $u_{t,i}^{(n)} = \beta_{x}(\nabla_{x} f^{(n)}(x^{(n)}_{t,i}, y^{(n)}_{t,i};\xi^{(n)}_{t,i}) - g^{(n)}_{x,t-1}+g_{x,t-1}) + (1-\beta_{x})u_{t-1}$, and $u_{t} = \beta_{x} g_{x,t} + (1-\beta_{x})u_{t-1} $, we obtain
\begin{align}
  u_{t} - u_{t,i}^{(n)} & = \beta_{x} (g_{x,t} - g_{x,t-1}) - \beta_{x}(\nabla_{x} f^{(n)}(x^{(n)}_{t,i}, y^{(n)}_{t,i};\xi^{(n)}_{t,i}) - g^{(n)}_{x,t-1}) \notag \\
    & = \beta_{x} \frac{1}{Np}\sum_{n'=1}^N\sum_{i'=0}^{p-1}\nabla_{x} f^{(n')}(x^{(n')}_{t,i'}, y^{(n')}_{t,i'}; \xi^{(n')}_{t,i'}) - \beta_{x} \frac{1}{Np}\sum_{n'=1}^N\sum_{i'=0}^{p-1} g^{(n')}_{x,t-1}\notag \\
    & \quad - \beta_{x}(\nabla_{x} f^{(n)}(x^{(n)}_{t,i}, y^{(n)}_{t,i};\xi^{(n)}_{t,i}) -  g^{(n)}_{x,t-1}) \notag  \ . 
\end{align}
Then, we obtain
\begin{align}
    & \quad \frac{1}{Np}\sum_{n=1}^N\sum_{i=0}^{p-1}\mathbb{E}[ \| u_{t} - u_{t,i}^{(n)}\| ] \notag \\
    & \leq \beta_{x} \frac{1}{Np}\sum_{n=1}^N\sum_{i=0}^{p-1}\mathbb{E}[ \|  \frac{1}{Np}\sum_{n'=1}^N\sum_{i'=0}^{p-1}\nabla_{x} f^{(n')}(x^{(n')}_{t,i'}, y^{(n')}_{t,i'}; \xi^{(n')}_{t,i'}) -  \frac{1}{Np}\sum_{n'=1}^N\sum_{i'=0}^{p-1} g^{(n')}_{x,t-1}\| ] \notag \\
    & \quad + \beta_{x}\frac{1}{Np}\sum_{n=1}^N\sum_{i=0}^{p-1}\mathbb{E}[ \| \nabla_{x} f^{(n)}(x^{(n)}_{t,i}, y^{(n)}_{t,i};\xi^{(n)}_{t,i}) -  g^{(n)}_{x,t-1}\|] \notag \\
    & = \beta_{x} \mathbb{E}[ \|  \frac{1}{Np}\sum_{n=1}^N\sum_{i=0}^{p-1}\nabla_{x} f^{(n)}(x^{(n)}_{t,i}, y^{(n)}_{t,i}; \xi^{(n)}_{t,i}) -  \frac{1}{Np}\sum_{n'=1}^N\sum_{i=0}^{p-1} g^{(n)}_{x,t-1}\| ] \notag \\
    & \quad + \beta_{x}\frac{1}{Np}\sum_{n=1}^N\sum_{i=0}^{p-1}\mathbb{E}[ \| \nabla_{x} f^{(n)}(x^{(n)}_{t,i}, y^{(n)}_{t,i};\xi^{(n)}_{t,i}) -  g^{(n)}_{x,t-1}\|] \notag \\
    & = \beta_x T_1 + \beta_x T_2 \ .
\end{align}
For $T_1$, we obtain
\begin{align}
    & T_1 =\mathbb{E}[ \|  \frac{1}{Np}\sum_{n=1}^N\sum_{i=0}^{p-1}\nabla_{x} f^{(n)}(x^{(n)}_{t,i}, y^{(n)}_{t,i}; \xi^{(n)}_{t,i}) -  \frac{1}{Np}\sum_{n'=1}^N\sum_{i=0}^{p-1} g^{(n)}_{x,t-1}\| ] \notag \\
    & \leq \mathbb{E}[ \|  \frac{1}{Np}\sum_{n=1}^N\sum_{i=0}^{p-1}(\nabla_{x} f^{(n)}(x^{(n)}_{t,i}, y^{(n)}_{t,i}; \xi^{(n)}_{t,i})  - \nabla_{x}f^{(n)}(x^{(n)}_{t,i}, y^{(n)}_{t,i} )) \|] \notag \\
    & \quad + \mathbb{E}[ \|  \frac{1}{Np}\sum_{n=1}^N\sum_{i=0}^{p-1} (\nabla_{x}  f^{(n)}(x^{(n)}_{t,i}, y^{(n)}_{t,i}) -\nabla_{x} f^{(n)}(x_{t}, y_{t})) \|] \notag \\
    & \quad + \mathbb{E}[ \|  \frac{1}{Np}\sum_{n=1}^N\sum_{i=0}^{p-1} ( \nabla_{x} f^{(n)}(x_{t}, y_{t}) - \nabla_{x} f^{(n)}(x_{t-1}, y_{t-1})) \|] \notag \\
    & \quad + \mathbb{E}[ \|  \frac{1}{Np}\sum_{n=1}^N\sum_{i=0}^{p-1} ( \nabla_{x} f^{(n)}(x_{t-1}, y_{t-1})  - \frac{1}{p}\sum_{i'=0}^{p-1} \nabla_{x} f^{(n)}(x_{t-1,i'}^{(n)}, y_{t-1,i'}^{(n)})\| ] \notag \\
    & \quad + \mathbb{E}[ \|  \frac{1}{Np}\sum_{n=1}^N\sum_{i=0}^{p-1} (  \frac{1}{p}\sum_{i'=0}^{p-1} \nabla_{x} f^{(n)}(x_{t-1,i'}^{(n)}, y_{t-1,i'}^{(n)})-  g^{(n)}_{x,t-1})\| ] \ . 
\end{align}
The first term can be bounded as follows:
\begin{align}
    & \quad \mathbb{E}[ \|  \frac{1}{Np}\sum_{n=1}^N\sum_{i=0}^{p-1}(\nabla_{x} f^{(n)}(x^{(n)}_{t,i}, y^{(n)}_{t,i}; \xi^{(n)}_{t,i})  - \nabla_{x}f^{(n)}(x^{(n)}_{t,i}, y^{(n)}_{t,i} )) \|]  = \frac{1}{Np}\mathbb{E}[ \|  \sum_{n=1}^N\sum_{i=0}^{p-1}\delta_{t,i}^{(n)}\|] \notag \\
    &  \overset{\scriptstyle \text{Lemma~\ref{lemma:zijian-liu-lemma}}}{\leq}  \frac{2\sqrt{2}}{Np}\mathbb{E}
    \left[ \left(\sum_{n=1}^N\sum_{i=0}^{p-1}\|  \delta_{t,i}^{(n)}\|^s\right)^{\frac{1}{s}}\right] \overset{\scriptstyle (a)}{\leq} \frac{2\sqrt{2}}{Np} \left(\sum_{n=1}^N\sum_{i=0}^{p-1}\mathbb{E}
    \left[\|  \delta_{t,i}^{(n)}\|^s\right]\right)^{\frac{1}{s}} \overset{\scriptstyle \text{Assumption~\ref{assumption:ht_variance}}}{\leq}\frac{2\sqrt{2}}{(Np)^{1-1/s}}\sigma \notag \\
    & \leq 2\sqrt{2}\sigma \ ,
\end{align}
where $(a)$ holds due to Hölder’s inequality, and the last step holds due to $s\in (1, 2]$, $p>1$,  and $N>1$

The last term can be bounded as follows:
\begin{align}
    & \quad \mathbb{E}[ \|  \frac{1}{Np}\sum_{n=1}^N\sum_{i=0}^{p-1} (  \frac{1}{p}\sum_{i'=0}^{p-1} \nabla_{x} f^{(n)}(x_{t-1,i'}^{(n)}, y_{t-1,i'}^{(n)})-  g^{(n)}_{x,t-1})\| ]  \notag \\
    & = \mathbb{E}[ \|  \frac{1}{Np}\sum_{n=1}^N\sum_{i=0}^{p-1} (  \frac{1}{p}\sum_{i'=0}^{p-1} \nabla_{x} f^{(n)}(x_{t-1,i'}^{(n)}, y_{t-1,i'}^{(n)})-  \frac{1}{p}\sum_{i'=0}^{p-1} \nabla_{x} f^{(n)}(x_{t-1,i'}^{(n)}, y_{t-1,i'}^{(n)}; \xi_{t-1,i'}^{(n)}))\| ]  \notag \\
    & = \mathbb{E}[ \|  \frac{1}{Np}\sum_{n=1}^N \sum_{i=0}^{p-1}(   \nabla_{x} f^{(n)}(x_{t-1,i}^{(n)}, y_{t-1,i}^{(n)})-  \nabla_{x} f^{(n)}(x_{t-1,i}^{(n)}, y_{t-1,i}^{(n)}; \xi_{t-1,i}^{(n)}))\| ]  \notag \\
    & \leq 2\sqrt{2}\sigma \ ,
\end{align}
where the last step is obtained as the proof of the first term. 
Then, we obtain
\begin{align}
    & T_1 \leq 4\sqrt{2}\sigma + 2(\eta_{x}+\eta_{y}) pL_f+ (\gamma_{x}+\gamma_{y}) L_f \ .
\end{align}
Similarly, we obtain
\begin{align}
    & T_2 \leq 4\sqrt{2}\sigma + 2(\eta_{x}+\eta_{y}) pL_f+ (\gamma_{x}+\gamma_{y}) L_f \ .
\end{align}
As a result, we have
\begin{align}\label{eq:local_grad_control}
    &  \frac{1}{Np}\sum_{n=1}^N\sum_{i=0}^{p-1}\mathbb{E}[ \| u_{t} - u_{t,i}^{(n)}\| ]  \leq 8\sqrt{2}\beta_{x}\sigma + 4\beta_{x} (\eta_{x}+\eta_{y}) pL_f +  2\beta_{x} (\gamma_{x}+\gamma_{y}) L_f      \ .
\end{align}

Summing up from $t=0$ to $T-1$, we obtain
\begin{align}
    \frac{1}{NpT}\sum_{t=0}^{T-1}\sum_{n=1}^N\sum_{i=0}^{p-1} \mathbb{E}[ \| u_{t} - u_{t,i}^{(n)}\| ] \leq 8\sqrt{2}\beta_{x}\sigma + 4\beta_{x}( \eta_{x}+\eta_{y}) pL_f +  2\beta_{x}(\gamma_{x}+\gamma_{y})L_f   \ .
\end{align}
Similarly, the second inequality in the lemma can be proved by following the same line of  reasoning. Thus, the proof is complete.
\end{proof}

\subsection{Proof of the Theorem~\ref{theorem:convergence-rate}}
\begin{proof}
    We have established an upper bound for $\frac{1}{T}\sum_{t=0}^{T-1}\mathbb{E}[\| \nabla \Phi({x}_{t})  \|] $ in Eq.~(\ref{eq:phi_sum}) as shown in Lemma~\ref{lemma:phi_sum}. Next, we substitute the results from Lemma~\ref{lemma:global_grad_error_x_y} and Lemma~\ref{lemma:local_grad_error_x_y} into  Eq.~(\ref{eq:phi_sum}): 
    \begin{align}
        & \frac{1}{T}\sum_{t=0}^{T-1}\mathbb{E}[\| \nabla \Phi({x}_{t})  \|] \leq \frac{(\Phi(x_{0})- \Phi^*) }{\gamma_{x}T}  + \frac{\Phi(x_{0}) - f(x_0,y_0)}{3\gamma_{x}T}  + \frac{2L_{\Phi}\gamma_{x}}{3}  + \frac{L_f\gamma_{x}(1+10\kappa)^2}{6} \notag \\
        & + \frac{10}{3}\Big(\frac{(\eta_{x}+\eta_{y})pL_f}{\beta_{x} T} +\frac{1}{\beta_{x} T}\frac{2\sqrt{2}\sigma }{(Np)^{1-1/s}}  + \frac{(\gamma_{x}+\gamma_{y})L_f}{\beta_{x}}  +  (\eta_{x}+\eta_{y}) pL_f + \frac{2\sqrt{2}\beta_{x}^{1-1/s}}{(Np)^{1-1/s}}\sigma \Big) \notag \\
        & +  \frac{20\kappa}{3}\Big( \frac{(\eta_{x}+\eta_{y})pL_f}{\beta_{y} T} +\frac{1}{\beta_{y} T}\frac{2\sqrt{2}\sigma }{(Np)^{1-1/s}}  + \frac{(\gamma_{x}+\gamma_{y})L_f}{\beta_{y}} +  (\eta_{x}+\eta_{y}) pL_f + \frac{2\sqrt{2}\beta_{y}^{1-1/s}}{(Np)^{1-1/s}}\sigma \Big) \notag \\
        &  + \frac{5}{3}\Big( 8\sqrt{2}\beta_{x}\sigma + 4\beta_{x}( \eta_{x}+\eta_{y}) pL_f +  2\beta_{x}(\gamma_{x}+\gamma_{y})L_f   \Big) \notag \\
        & + \frac{10\kappa}{3}\Big( 8\sqrt{2}\beta_{y}\sigma + 4\beta_{y}( \eta_{x}+\eta_{y}) pL_f +  2\beta_{y}(\gamma_{x}+\gamma_{y})L_f    \Big) \ . 
    \end{align}
    By applying the conditions $\beta_{x} =O(\beta)< 1$, $\beta_{y}=O(\beta)<1$, $s\in(1,2]$, $L_{\Phi} = O(\kappa)$, $\gamma_{x} = \frac{\gamma_{y}}{10\kappa}$, we simplify the above inequality as follows:
    \begin{align}
        & \frac{1}{T}\sum_{t=0}^{T-1}\mathbb{E}[\| \nabla \Phi({x}_{t})  \|] \leq O\left(\frac{(\Phi(x_{0})- \Phi^*) }{\gamma_{x}T}\right)  + O\left(\frac{\Phi(x_{0}) - f(x_0,y_0)}{\gamma_{x}T}\right) + O\left(\kappa^2\gamma_{x}\right)  \notag \\
        & + O\left(\frac{\kappa(\eta_{x}+\eta_{y})p}{\beta T}\right) + O\left(\frac{\kappa\sigma }{\beta T(Np)^{1-1/s}}\right) +  O\left(\frac{\kappa^2\gamma_{x}}{\beta}\right) + O\left(\kappa(\eta_{x}+\eta_{y})p\right) \notag \\
        & + O\left(\frac{\kappa\beta^{1-1/s}\sigma }{(Np)^{1-1/s}}\right)  + O\left(\kappa\beta\sigma \right) + O\left( \kappa\beta(\eta_{x}+\eta_{y})p\right) + O\left( \kappa^2\beta \gamma_{x}\right)   \ .
    \end{align}
    By setting
    \begin{align}
        & \gamma_{x} = O\left(\frac{(Np)^{1/4}}{\kappa T^{3/4}}\right) \ , \quad \beta = O\left(\frac{(Np)^{1/2}}{T^{1/2}}\right) \ ,  \eta_x = O\left(\frac{1}{p\sqrt{T}}\right) \ , \quad \eta_y =   O\left(\frac{1}{p\sqrt{T}}\right) \ ,
    \end{align}
    we obtain
    \begin{align}
        \frac{1}{T}\sum_{t=0}^{T-1}\mathbb{E}[\| \nabla \Phi({x}_{t})  \|] \leq O\left(\frac{\kappa}{(TNp)^{1/4}}\right) + O\left(\frac{\kappa \sigma}{(TNp)^{\frac{s-1}{2s}}}\right) \ . 
    \end{align}
   
\end{proof}
\section{Appendix: Muon Update}\label{app:muon}

\begin{algorithm*}[ht]
	\caption{FedMuon-DA}
	\label{alg-FedMuon-DA}
    \small
	\begin{algorithmic}[1]
		\REQUIRE initial model $X_0$, $Y_0$, global learning rates $\gamma_x$, $\gamma_y$, local learning rates $\eta_x$, $\eta_y$, momentum parameter $\beta_x$, $\beta_y$, local updates rounds $P$, and communication rounds $T$.  \\
		\vspace{2mm}
		\FOR{$t=0,\cdots, T-1$} 		
		\begin{tcolorbox}[colback=lightgray,colframe=black!40,boxrule=0.3pt,arc=2pt,left=2pt,right=2pt,top=1pt,bottom=1pt]
        \FOR{each client $n$}
		\STATE Initialize local model $X^{(n)}_{t,0} = X_{t}$, $Y^{(n)}_{t,0} = Y_{t}$.

        \FOR{$i=0,\cdots, p-1$}
		\STATE  Compute local momentum: \\
        $\ U^{(n)}_{t,i} = \beta_x(\nabla_X f^{(n)}(X^{(n)}_{t,i}, Y^{(n)}_{t,i};\xi^{(n)}_{t,i}) +   g_{X,t-1} - g^{(n)}_{X,t-1}) + (1-\beta_x)U_{t-1}$ \ , \\
        $\ V^{(n)}_{t,i} = \beta_y(\nabla_Y f^{(n)}(X^{(n)}_{t,i}, Y^{(n)}_{t,i};\xi^{(n)}_{t,i}) +  g_{Y,t-1} - g^{(n)}_{Y,t-1}) + (1-\beta_y)V_{t-1}$ \ .
        \STATE  $\ $ Orthonormalize $U^{(n)}_{t,i}$ with Newton–Schulz approach: $(P^{(n)}_{t,i}, \Sigma^{(n)}_{t,i}, Q^{(n)}_{t,i})=\text{SVD}(U^{(n)}_{t,i})$ \ , 
	   \STATE $\ $ Update variable $X^{(n)}_{t,i}$: 	$X^{(n)}_{t+1,i} =  X^{(n)}_{t,i} -\eta_{x} P^{(n)}_{t,i}(Q^{(n)}_{t,i})^T$ \ , 
       \STATE $\ $ Orthonormalize $V^{(n)}_{t,i}$ with Newton–Schulz approach:  $(R^{(n)}_{t,i}, \Sigma^{(n)}_{t,i}, S^{(n)}_{t,i})=\text{SVD}(V^{(n)}_{t,i})$ \ , 
	   \STATE $\ $ Update variable $Y^{(n)}_{t,i}$: $Y^{(n)}_{t+1,i} =  Y^{(n)}_{t,i} +\eta_{y} R^{(n)}_{t,i}(S^{(n)}_{t,i})^T$ .
		\ENDFOR
        \STATE Aggregate local control variates: \\
        $ \  g_{X,t}^{(n)} = \frac{1}{p}\sum_{i=0}^{p-1}\nabla_X f^{(n)}(X^{(n)}_{t,i}, Y^{(n)}_{t,i}; \xi^{(n)}_{t,i})$ \ , 
        $\  g_{Y,t}^{(n)} = \frac{1}{p}\sum_{i=0}^{p-1}\nabla_Y f^{(n)}(X^{(n)}_{t,i}, Y^{(n)}_{t,i}; \xi^{(n)}_{t,i})$ \ .
		\ENDFOR
        \end{tcolorbox}
        \textbf{Central Server}: \\
        \begin{tcolorbox}[colback=lightgray,colframe=black!40,boxrule=0.3pt,arc=2pt,left=2pt,right=2pt,top=1pt,bottom=1pt]
        \STATE Aggregate global control variates: 
        $\ g_{X,t} = \frac{1}{N}\sum_{n=1}^Ng_{X,t}^{(n)}$ \ , 
        $\ g_{Y,t} = \frac{1}{N}\sum_{n=1}^Ng_{Y,t}^{(n)}$ \ .
        \STATE Global update: 
        $\ X_{t+1} = X_{t} + \frac{\gamma_x}{\eta_x Np}\sum_{n=1}^N( X^{(n)}_{t,p} - X_{t})$ \ , 
        $\  Y_{t+1} = Y_{t} + \frac{\gamma_y}{\eta_y Np}\sum_{n=1}^N(Y^{(n)}_{t,p} - Y_{t} )$ \ .
        \STATE Update global momentum: 
        $\  U_{t}  =   \beta_x g_{X,t} + (1-\beta_x)U_{t-1} $ \ ,  
        $\ V_{t}  = \beta_y g_{Y,t} + (1-\beta_y)V_{t-1} $ \ . 
        \end{tcolorbox}
		\ENDFOR
	\end{algorithmic}
\end{algorithm*}

\begin{lemma}\label{lemma_muon:consensus_x_y}
    Given Assumptions~\ref{assumption:smooth}-\ref{assumption:ht_variance}, the following inequalities hold:
    \begin{align}
        &  \frac{1}{Np}\sum_{n=1}^N\sum_{i=0}^{p-1} \| X_{t,i}^{(n)} - X_{t}\|_F \leq \eta_{x}p\sqrt{n_x} \ , \quad \frac{1}{Np}\sum_{n=1}^N\sum_{i=0}^{p-1} \| Y_{t,i}^{(n)} - Y_{t}\|_F \leq \eta_{y}p\sqrt{n_y} \ .
    \end{align}
\end{lemma}

\begin{proof}
    \begin{align}
		& \quad \|X_{t,i}^{(n)} - X_{t} \|_F \leq \sum_{j=0}^{i-1}\|X_{t,j+1}^{(n)} - X_{t,j} \|_F \leq \eta_{x}\sum_{j=0}^{i-1}\|P^{(n)}_{t,j}(Q^{(n)}_{t,j})^T\|_F \leq \eta_{x}p\sqrt{n_x}  \ , 
	\end{align}
	where the last step holds due to $\| P^{(n)}_{t,i}(Q^{(n)}_{t,i})^T \|_F\leq \sqrt{n_x}$. Taking the average over all $n$ and $i$ completes the proof. The argument for $y$ is identical.
\end{proof}

\begin{lemma}\label{lemma_muon:phi_smooth}
    Given Assumptions~\ref{assumption:smooth}-\ref{assumption:ht_variance}, the following inequality holds:
    \begin{align}
        \mathbb{E}[\Phi(X_{t+1})] - \mathbb{E}[\Phi(X_{t})] & - \gamma_{x} \mathbb{E}[\|  \nabla \Phi(X_{t})\|_F]  +2 \gamma_{x}\kappa \sqrt{n_x}\mathbb{E}[\|\nabla_{Y} f(X_{t}, Y_{t}) \|_F]  +2 \gamma_{x} \sqrt{n_x}\mathbb{E}[\| \nabla_{X} f(X_{t}, Y_{t}) -U_t\|_F]  \notag \\
 	    & \quad   + 2\gamma_{x} \sqrt{n_x}\frac{1}{Np}\sum_{n=1}^N\sum_{i=0}^{p-1}\mathbb{E}[\|U_{t} -U_{t,i}^{(n)}\|_F]   + \frac{L_{\Phi}n_x\gamma_{x}^2}{2} \ .
    \end{align}
\end{lemma}

\begin{proof}
    Due to the $L_{\Phi}$-smoothness of $\Phi(\cdot)$, we have
      \begin{align}\label{eq_muon:l_phi}
    	& \quad \mathbb{E}[\Phi(X_{t+1})]    \leq \mathbb{E}[\Phi(X_{t})] + \mathbb{E}[\langle \nabla \Phi(X_{t}), X_{t+1} - X_{t} \rangle] + \frac{L_{\Phi}}{2}\mathbb{E}[\| X_{t+1} - X_{t} \|_F^2] \notag \\
    	& \overset{\scriptstyle (a)}{\leq}  \mathbb{E}[\Phi(X_{t})] - \gamma_{x} \mathbb{E}[\langle \nabla \Phi(X_{t}), \frac{1}{Np}\sum_{n=1}^N\sum_{i=0}^{p-1} P_{t,i}^{(n)}(Q_{t,i}^{(n)})^T \rangle] + \frac{L_{\Phi}n_x\gamma_{x}^2}{2} \notag \\
 	    &= \mathbb{E}[\Phi(X_{t})] - \gamma_{x}  \frac{1}{Np}\sum_{n=1}^N\sum_{i=0}^{p-1} \mathbb{E}[\langle \nabla \Phi(X_{t}) -U_{t,i}^{(n)}, P_{t,i}^{(n)}(Q_{t,i}^{(n)})^T \rangle]  - \gamma_{x} \frac{1}{Np}\sum_{n=1}^N\sum_{i=0}^{p-1}\mathbb{E}[\langle U_{t,i}^{(n)},  P_{t,i}^{(n)}(Q_{t,i}^{(n)})^T \rangle]  + \frac{L_{\Phi}n_x\gamma_{x}^2}{2} \notag \\
 	    &\overset{\scriptstyle (b)}{\leq} \mathbb{E}[\Phi(X_{t})] + \gamma_{x} \sqrt{n_x}\frac{1}{Np}\sum_{n=1}^N\sum_{i=0}^{p-1}\mathbb{E}[\| \nabla \Phi(X_{t}) -U_{t,i}^{(n)}\|_F]  - \gamma_{x}\frac{1}{Np}\sum_{n=1}^N\sum_{i=0}^{p-1} \mathbb{E}[\| U_{t,i}^{(n)}\|_*]  + \frac{L_{\Phi}n_x\gamma_{x}^2}{2} \notag \\
 	    &\leq \mathbb{E}[\Phi(X_{t})] + \gamma_{x} \sqrt{n_x}\mathbb{E}[\| \nabla \Phi(X_{t}) -U_{t}\|_F]  + \gamma_{x} \sqrt{n_x}\frac{1}{Np}\sum_{n=1}^N\sum_{i=0}^{p-1}\mathbb{E}[\|U_{t} -U_{t,i}^{(n)}\|_F] \notag \\
 	    & \quad  - \gamma_{x}\frac{1}{Np}\sum_{n=1}^N\sum_{i=0}^{p-1} \mathbb{E}[\| U_{t,i}^{(n)}\|_*]  + \frac{L_{\Phi}n_x\gamma_{x}^2}{2} \notag \\
 	    & \overset{\scriptstyle (c)}{\leq} \mathbb{E}[\Phi(X_{t})] + \gamma_{x} \sqrt{n_x}\mathbb{E}[\| \nabla \Phi(X_{t}) -U_{t}\|_F]  +2 \gamma_{x} \sqrt{n_x}\frac{1}{Np}\sum_{n=1}^N\sum_{i=0}^{p-1}\mathbb{E}[\|U_{t} -U_{t,i}^{(n)}\|_F] \notag \\
 	    & \quad  - \gamma_{x} \mathbb{E}[\| U_{t}\|_*]  + \frac{L_{\Phi}n_x\gamma_{x}^2}{2} \notag \\
 	    &\overset{\scriptstyle (d)}{\leq}  \mathbb{E}[\Phi(X_{t})] - \gamma_{x} \mathbb{E}[\|\nabla \Phi(X_{t})\|_F]    + \frac{L_{\Phi}n_x\gamma_{x}^2}{2} \notag \\
 	    & \quad   +2 \gamma_{x} \sqrt{n_x}\mathbb{E}[\| \nabla \Phi(X_{t}) -U_{t}\|_F]  + 2\gamma_{x} \sqrt{n_x}\frac{1}{Np}\sum_{n=1}^N\sum_{i=0}^{p-1}\mathbb{E}[\|U_{t} -U_{t,i}^{(n)}\|_F] \notag \\
 	     &\leq \mathbb{E}[\Phi(X_{t})] - \gamma_{x} \mathbb{E}[\|  \nabla \Phi(X_{t})\|_F]   + \frac{L_{\Phi}n_x\gamma_{x}^2}{2} +2 \gamma_{x} \sqrt{n_x}\mathbb{E}[\| \nabla \Phi(X_{t}) -\nabla_{X} f(X_{t}, Y_{t}) \|_F]  \notag \\
 	     & \quad   +2 \gamma_{x} \sqrt{n_x}\mathbb{E}[\| \nabla_{X} f(X_{t}, Y_{t}) -U_t\|_F] + 2\gamma_{x} \sqrt{n_x}\frac{1}{Np}\sum_{n=1}^N\sum_{i=0}^{p-1}\mathbb{E}[\|U_{t} -U_{t,i}^{(n)}\|_F] \notag \\
 	     &\overset{\scriptstyle (e)}{\leq} \mathbb{E}[\Phi(X_{t})] - \gamma_{x} \mathbb{E}[\|  \nabla \Phi(X_{t})\|_F]   + \frac{L_{\Phi}n_x\gamma_{x}^2}{2}  +2 \gamma_{x}\kappa \sqrt{n_x}\mathbb{E}[\|\nabla_{Y} f(X_{t}, Y_{t}) \|_F]   \notag \\
 	     & \quad +2 \gamma_{x} \sqrt{n_x}\mathbb{E}[\| \nabla_{X} f(X_{t}, Y_{t}) -U_t\|_F]   + 2\gamma_{x} \sqrt{n_x}\frac{1}{Np}\sum_{n=1}^N\sum_{i=0}^{p-1}\mathbb{E}[\|U_{t} -U_{t,i}^{(n)}\|_F] \ , 
    \end{align}
    where  $(a)$ follows from 
    \begin{align}
        \| X_{t+1} - X_{t} \|_F^2=\gamma^2_{x}\| \frac{1}{Np}\sum_{n=1}^N\sum_{i=0}^{p-1} P_{t,i}^{(n)}(Q_{t,i}^{(n)})^T \|_F^2\leq \gamma^2_{x}\frac{1}{Np}\sum_{n=1}^N\sum_{i=0}^{p-1}\|  P_{t,i}^{(n)}(Q_{t,i}^{(n)})^T \|_F^2\leq n_x \gamma^2_{x} \ , \notag 
    \end{align}
    $(b)$ follows from $\langle U_{t,i}^{(n)},  P_{t,i}^{(n)}(Q_{t,i}^{(n)})^T \rangle=\| U_{t,i}^{(n)}\|_*$ and 
    \begin{align}
        -\langle \nabla \Phi(x_{t}) -U_{t,i}^{(n)}, P_{t,i}^{(n)}(Q_{t,i}^{(n)})^T \rangle\leq \| \nabla \Phi(x_{t}) -U_{t,i}^{(n)}\|_F\| P_{t,i}^{(n)}(Q_{t,i}^{(n)})^T \|_F\leq \sqrt{n_x}\| \nabla \Phi(x_{t}) -U_{t,i}^{(n)}\|_F \ ,  \notag 
    \end{align} 
    $(c)$ follows from 
    \begin{align}
        \|U_t\|_*\leq \|U_{t,i}^{(n)} - U_t\|_* + \|U_{t,i}^{(n)}\|_*\leq \sqrt{n}\|U_{t,i}^{(n)} - U_t\|_F + \|U_{t,i}^{(n)}\|_* \ , \notag 
    \end{align}
    $(d)$ follows from 
    \begin{align}
        \|\nabla \Phi(X_{t})\|_F\leq\|\nabla \Phi(X_{t})\|_*\leq \|\nabla \Phi(X_{t}) - U_t\|_* + \|U_t\|_*\leq \sqrt{n_x}\|\nabla \Phi(X_{t}) - U_t\|_F + \|U_t\|_* \ , \notag 
    \end{align}
    and $(e)$ follows from 
    \begin{align}\label{eq_muon:phi_u}
        &  \mathbb{E}[\|\nabla \Phi(X_{t}) - \nabla_{X} f(X_{t}, Y_{t}) \|]  \leq L_f\mathbb{E}[\|Y^*(X_{t}) - Y_{t}\|] \leq \kappa\mathbb{E}[\|\nabla_{Y} f(X_{t}, Y_{t}) \|]  \ , 
    \end{align}
    where the last step holds due to the inequality $\|Y^*(x) - Y\| \leq \frac{1}{\mu}\|\nabla_Y f(X,Y)\|$, as established in Appendix A of \cite{karimi2016linear}, and $\kappa = L_f/\mu$.
\end{proof}

\begin{lemma}\label{lemma_muon:phi_f}
    Given Assumptions~\ref{assumption:smooth}-\ref{assumption:ht_variance}, the following inequality holds:
    \begin{align}
      &  \mathbb{E}[f(X_{t}, Y_{t})] - \mathbb{E}[f(X_{t+1}, Y_{t+1})] \leq   \gamma_{x}\sqrt{n_x}\mathbb{E}[\|  \nabla \Phi(X_t)  \|_F] +  (\gamma_{x} \sqrt{n_x}\kappa -  \gamma_{y}) \mathbb{E}[\|  \nabla_{Y} f(X_{t}, Y_{t})   \|_F]   \notag \\
      & + 2\gamma_{x} \sqrt{n_x}\mathbb{E}[\| \nabla_{X} f(X_{t}, Y_{t}) - U_{t} \|_F] + 2\gamma_{y} \sqrt{n_y}\mathbb{E}[\| \nabla_{Y} f(X_{t}, Y_{t}) -V_{t}\|_F]  \\
      &  +2 \gamma_{y} \sqrt{n_y}\frac{1}{Np}\sum_{n=1}^N\sum_{i=0}^{p-1}\mathbb{E}[\|V_{t} -V_{t,i}^{(n)}\|_F]    +2 \gamma_{x} \sqrt{n_x}\frac{1}{Np}\sum_{n=1}^N\sum_{i=0}^{p-1}\mathbb{E}[\| U_{t} - U_{t,i}^{(n)} \|_F] + \frac{L_f}{2}(\gamma_{x}^2{n_x}+\gamma_{y}^2n_y +2n_x\gamma_{x}\gamma_{y}) \ . \notag 
    \end{align}
\end{lemma}

\begin{proof}
Following Eq.~(\ref{eq_muon:l_phi}), due to the smoothness of $f$ regarding $y$, we obtain
\begin{align}
    & \quad \mathbb{E}[f(X_{t+1}, Y_{t})] \leq \mathbb{E}[f(X_{t+1}, Y_{t+1})] - \mathbb{E}[\langle \nabla_Y f(X_{t+1}, Y_{t}) , Y_{t+1} - Y_{t} \rangle] + \frac{L_f}{2}\mathbb{E}[\|Y_{t+1} - Y_{t}\|_F^2] \notag \\
    & \overset{\scriptstyle (a)}{\leq}  \mathbb{E}[f(X_{t+1}, Y_{t+1})] - \gamma_{y} \mathbb{E}[\langle \nabla_{Y} f(X_{t+1}, Y_{t}), \frac{1}{Np}\sum_{n=1}^N\sum_{i=0}^{p-1} R_{t,i}^{(n)} (S_{t,i}^{(n)})^T \rangle] + \frac{L_{f}n_y\gamma_{y}^2}{2} \notag \\
    & =  \mathbb{E}[f(X_{t+1}, Y_{t+1})] - \gamma_{y} \frac{1}{Np}\sum_{n=1}^N\sum_{i=0}^{p-1}\mathbb{E}[\langle \nabla_{Y} f(X_{t+1}, Y_{t}) - V_{t,i}^{(n)},  R_{t,i}^{(n)} (S_{t,i}^{(n)})^T \rangle]  \notag \\
    & \quad - \gamma_{y} \frac{1}{Np}\sum_{n=1}^N\sum_{i=0}^{p-1}\mathbb{E}[\langle V_{t,i}^{(n)},  R_{t,i}^{(n)} (S_{t,i}^{(n)})^T \rangle] + \frac{L_{f}n_y\gamma_{y}^2}{2} \notag \\
    & \leq   \mathbb{E}[f(X_{t+1}, Y_{t+1})] + \gamma_{y}\sqrt{n_y} \frac{1}{Np}\sum_{n=1}^N\sum_{i=0}^{p-1}\mathbb{E}[\| \nabla_{Y} f(X_{t+1}, Y_{t}) - V_{t,i}^{(n)}\|_F]  \notag \\
    & \quad - \gamma_{y} \frac{1}{Np}\sum_{n=1}^N\sum_{i=0}^{p-1}\mathbb{E}[\| V_{t,i}^{(n)}\|_*] + \frac{L_{f}n_y\gamma_{y}^2}{2} \notag \\
    &\leq  \mathbb{E}[f(X_{t+1}, Y_{t+1})] + \gamma_{y} \sqrt{n_y}\mathbb{E}[\| \nabla_{Y} f(X_{t+1}, Y_{t}) -V_{t}\|_F]  +2 \gamma_{y} \sqrt{n_y}\frac{1}{Np}\sum_{n=1}^N\sum_{i=0}^{p-1}\mathbb{E}[\|V_{t} -V_{t,i}^{(n)}\|_F] \notag \\
    & \quad  - \gamma_{y} \mathbb{E}[\| V_{t}\|_*]  + \frac{L_{f}n_y\gamma_{y}^2}{2} \notag \\
    &\leq  \mathbb{E}[f(X_{t+1}, Y_{t+1})] + \gamma_{y} \sqrt{n_y}\mathbb{E}[\| \nabla_{Y} f(X_{t+1}, Y_{t}) -\nabla_{Y} f(X_{t}, Y_{t}) \|_F]+ \gamma_{y} \sqrt{n_y}\mathbb{E}[\| \nabla_{Y} f(X_{t}, Y_{t}) -V_{t}\|_F]   \notag \\
    & \quad +2 \gamma_{y} \sqrt{n_y}\frac{1}{Np}\sum_{n=1}^N\sum_{i=0}^{p-1}\mathbb{E}[\|V_{t} -V_{t,i}^{(n)}\|_F]   - \gamma_{y} \mathbb{E}[\| V_{t}\|_*]  + \frac{L_{f}n_y\gamma_{y}^2}{2} \notag \\
    &\overset{\scriptstyle (b)}{\leq}   \mathbb{E}[f(X_{t+1}, Y_{t+1})] + \gamma_{y} \sqrt{n_y}L_f\mathbb{E}[\| X_{t+1}-X_{t}\|_F]+ 2\gamma_{y} \sqrt{n_y}\mathbb{E}[\| \nabla_{Y} f(X_{t}, Y_{t}) -V_{t}\|_F]   \notag \\
    & \quad +2 \gamma_{y} \sqrt{n_y}\frac{1}{Np}\sum_{n=1}^N\sum_{i=0}^{p-1}\mathbb{E}[\|V_{t} -V_{t,i}^{(n)}\|_F]   - \gamma_{y} \mathbb{E}[\|  \nabla_{Y} f(X_{t}, Y_{t})\|_F]  + \frac{L_{f}n_y\gamma_{y}^2}{2} \notag \\
    &\leq  \mathbb{E}[f(X_{t+1}, Y_{t+1})] + 2\gamma_{y} \sqrt{n_y}\mathbb{E}[\| \nabla_{Y} f(X_{t}, Y_{t}) -V_{t}\|_F]   \notag \\
    & \quad +2 \gamma_{y} \sqrt{n_y}\frac{1}{Np}\sum_{n=1}^N\sum_{i=0}^{p-1}\mathbb{E}[\|V_{t} -V_{t,i}^{(n)}\|_F]   - \gamma_{y} \mathbb{E}[\|  \nabla_{Y} f(X_{t}, Y_{t})\|_F]  + \frac{L_{f}\gamma_{y}(n_y\gamma_{y} +2n_x\gamma_{x} )}{2} \ , 
\end{align}
$(a)$ follows from 
\begin{align}
    \|Y_{t+1} - Y_{t}\|_F^2 = \gamma_{y}^2\| \frac{1}{Np}\sum_{n=1}^N\sum_{i=0}^{p-1} R_{t,i}^{(n)}(S_{t,i}^{(n)})^T \|_F^2  \leq \gamma_{y}^2\frac{1}{Np}\sum_{n=1}^N\sum_{i=0}^{p-1}\|  R_{t,i}^{(n)}(S_{t,i}^{(n)})^T \|_F^2\leq n_y \gamma^2_{y} \ , \notag 
\end{align}
$(b)$ follows from 
\begin{align}
    \|\nabla_y f(X_{t}, Y_{t})\|_F \leq \| \nabla_y f(X_{t}, Y_{t})\|_* \leq \|\nabla_y f(X_{t}, Y_{t}) - V_{t}\|_* + \|V_{t}\|_* \leq \sqrt{n_y}\|\nabla_y f(X_{t}, Y_{t}) - V_{t}\|_F + \|V_{t}\|_* \ . \notag 
\end{align}

Similarly, due to the smoothness of $f$ regarding $x$, we obtain
\begin{align}
    & \quad \mathbb{E}[f(X_{t}, Y_{t})] \leq \mathbb{E}[f(X_{t+1},Y_{t})] - \mathbb{E}[\langle \nabla_{X} f(X_{t}, Y_{t}), X_{t+1} - X_{t} \rangle] + \frac{L_f}{2}\mathbb{E}[\|X_{t+1} - X_{t}\|_F^2] \notag \\
     & \leq \mathbb{E}[f(X_{t+1},Y_{t})] + \gamma_{x}\mathbb{E}[\langle \nabla_{X} f(X_{t}, Y_{t}), \frac{1}{Np}\sum_{n=1}^N\sum_{i=0}^{p-1} P_{t,i}^{(n)}(Q_{t,i}^{(n)})^T  \rangle] + \frac{L_f{n_x}\gamma_{x}^2}{2} \notag \\
     & \leq \mathbb{E}[f(X_{t+1},Y_{t})] + \gamma_{x} \frac{1}{Np}\sum_{n=1}^N\sum_{i=0}^{p-1}\mathbb{E}[\langle \nabla_{X} f(X_{t}, Y_{t}) - U_{t,i}^{(n)} ,  P_{t,i}^{(n)}(Q_{t,i}^{(n)})^T  \rangle] \notag \\
      & \quad + \gamma_{x} \frac{1}{Np}\sum_{n=1}^N\sum_{i=0}^{p-1}\mathbb{E}[\langle U_{t,i}^{(n)} , P_{t,i}^{(n)}(Q_{t,i}^{(n)})^T  \rangle] + \frac{L_f{n_x}\gamma_{x}^2}{2} \notag \\
      & \leq \mathbb{E}[f(X_{t+1},Y_{t})] + \gamma_{x} \sqrt{n_x}\frac{1}{Np}\sum_{n=1}^N\sum_{i=0}^{p-1}\mathbb{E}[\| \nabla_{X} f(X_{t}, Y_{t}) - U_{t,i}^{(n)} \|_F] \notag \\
      & \quad + \gamma_{x} \frac{1}{Np}\sum_{n=1}^N\sum_{i=0}^{p-1}\mathbb{E}[\| U_{t,i}^{(n)} \|_*] + \frac{L_f{n_x}\gamma_{x}^2}{2} \notag \\
       & \leq \mathbb{E}[f(X_{t+1},Y_{t})] + \gamma_{x} \sqrt{n_x}\mathbb{E}[\| \nabla_{X} f(X_{t}, Y_{t}) - U_{t} \|_F]  + \gamma_{x} \sqrt{n_x}\frac{1}{Np}\sum_{n=1}^N\sum_{i=0}^{p-1}\mathbb{E}[\| U_{t} - U_{t,i}^{(n)} \|_F]\notag \\
      & \quad + \gamma_{x} \frac{1}{Np}\sum_{n=1}^N\sum_{i=0}^{p-1}\mathbb{E}[\| U_{t,i}^{(n)} \|_*] + \frac{L_f{n_x}\gamma_{x}^2}{2} \notag \\
       & \leq \mathbb{E}[f(X_{t+1},Y_{t})] + \gamma_{x} \sqrt{n_x}\mathbb{E}[\| \nabla_{X} f(X_{t}, Y_{t}) - U_{t} \|_F]  + \gamma_{x} \sqrt{n_x}\frac{1}{Np}\sum_{n=1}^N\sum_{i=0}^{p-1}\mathbb{E}[\| U_{t} - U_{t,i}^{(n)} \|_F] + \frac{L_f{n_x}\gamma_{x}^2}{2}\notag \\
      & \quad + \gamma_{x} \frac{1}{Np}\sum_{n=1}^N\sum_{i=0}^{p-1}\mathbb{E}[\| U_{t,i}^{(n)}  - U_{t} \|_*]   +  \gamma_{x} \frac{1}{Np}\sum_{n=1}^N\sum_{i=0}^{p-1}\mathbb{E}[\| U_{t} - \nabla_{X} f(X_{t}, Y_{t}) \|_*]  \notag \\
      & \quad +  \gamma_{x} \frac{1}{Np}\sum_{n=1}^N\sum_{i=0}^{p-1}\mathbb{E}[\|  \nabla_{X} f(X_{t}, Y_{t}) -\nabla \Phi(X_t)  \|_*]  + \gamma_{x} \frac{1}{Np}\sum_{n=1}^N\sum_{i=0}^{p-1}\mathbb{E}[\|  \nabla \Phi(X_t)  \|_*]  \notag \\
      & \leq \mathbb{E}[f(X_{t+1},Y_{t})] + 2\gamma_{x} \sqrt{n_x}\mathbb{E}[\| \nabla_{X} f(X_{t}, Y_{t}) - U_{t} \|_F]  +2 \gamma_{x} \sqrt{n_x}\frac{1}{Np}\sum_{n=1}^N\sum_{i=0}^{p-1}\mathbb{E}[\| U_{t} - U_{t,i}^{(n)} \|_F] + \frac{L_f{n_x}\gamma_{x}^2}{2}\notag \\
      & \quad +  \gamma_{x} \sqrt{n_x}\kappa \mathbb{E}[\|  \nabla_{Y} f(X_{t}, Y_{t})   \|_F]  + \gamma_{x}\sqrt{n_x}\mathbb{E}[\|  \nabla \Phi(X_t)  \|_F]   \ , 
\end{align}
where $(a)$ follows from 
\begin{align}
    \mathbb{E}[\| {U}_{t}\|] \leq \mathbb{E}[\| {U}_{t} - \nabla \Phi({X})\|] + \mathbb{E}[\|\nabla \Phi({X}) \|]  \overset{\scriptstyle \text{Eq.~(\ref{eq_muon:phi_u})}}{\leq}\kappa\mathbb{E}[\|\nabla_{y} f(X_{t}, Y_{t}) \|] + \mathbb{E}[\|\nabla_{x} f(X_{t}, Y_{t}) - U_{t} \|] + \mathbb{E}[\|\nabla \Phi({X}) \|] \ . \notag 
\end{align} 
By combining the above two inequalities, we obtain
\begin{align}
    & \mathbb{E}[f(X_{t}, Y_{t})] - \mathbb{E}[f(X_{t+1}, Y_{t+1})] \leq   \gamma_{x}\sqrt{n_x}\mathbb{E}[\|  \nabla \Phi(X_t)  \|_F] +  (\gamma_{x} \sqrt{n_x}\kappa -  \gamma_{y}) \mathbb{E}[\|  \nabla_{Y} f(X_{t}, Y_{t})   \|_F]   \notag \\
    & + 2\gamma_{x} \sqrt{n_x}\mathbb{E}[\| \nabla_{X} f(X_{t}, Y_{t}) - U_{t} \|_F] + 2\gamma_{y} \sqrt{n_y}\mathbb{E}[\| \nabla_{Y} f(X_{t}, Y_{t}) -V_{t}\|_F]  \\
    &  +2 \gamma_{y} \sqrt{n_y}\frac{1}{Np}\sum_{n=1}^N\sum_{i=0}^{p-1}\mathbb{E}[\|V_{t} -V_{t,i}^{(n)}\|_F]    +2 \gamma_{x} \sqrt{n_x}\frac{1}{Np}\sum_{n=1}^N\sum_{i=0}^{p-1}\mathbb{E}[\| U_{t} - U_{t,i}^{(n)} \|_F] + \frac{L_f}{2}(\gamma_{x}^2{n_x}+\gamma_{y}^2n_y +2n_x\gamma_{x}\gamma_{y}) \ . \notag 
\end{align}
The proof is complete by applying Lemma~\ref{lemma_muon:phi_smooth}.
\end{proof}

The following three lemmas are similar to Lemma~\ref{lemma:phi_sum}, Lemma~\ref{lemma:global_grad_error_x_y} and Lemma~\ref{lemma:local_grad_error_x_y}, but stated in matrix form; their proofs are omitted.

\begin{lemma}\label{lemma_muon:phi_sum}
    Given Assumptions~\ref{assumption:smooth}-\ref{assumption:ht_variance}, by setting $\gamma_{x} = \frac{\gamma_{y}}{10\kappa}$, the following inequality holds:
    \begin{align}\label{eq_muon:phi_sum}
        & \frac{1}{T}\sum_{t=0}^{T-1}\mathbb{E}[\| \nabla \Phi({X}_{t})  \|] \leq \frac{(\Phi(X_{0})- \Phi^*) }{\gamma_{x}T} + \frac{\Phi(X_{0}) - f(X_0,Y_0)}{3\gamma_{x}T} + \frac{2L_{\Phi}n_{x}\gamma_{x}}{3}  + \frac{L_f\gamma_{x}}{6} (n_{x}+100\kappa^2n_{y}+20\kappa n_{x})\notag \\
        & + \frac{10\sqrt{n_{x}}}{3}\frac{1}{T}\sum_{t=0}^{T-1}\mathbb{E}[\|\nabla_{X} f(X_{t}, Y_{t}) - U_{t} \|_F] +  \frac{20\kappa\sqrt{n_{y}}}{3}\frac{1}{T}\sum_{t=0}^{T-1}\mathbb{E}[\| \nabla_{Y} f(X_{t}, Y_{t}) - V_{t}\|_F] \notag \\
        &  + \frac{10\sqrt{n_{x}}}{3NpT}\sum_{t=0}^{T-1}\sum_{n=1}^N\sum_{i=0}^{p-1} \mathbb{E}[ \| U_{t} - U_{t,i}^{(n)}\|_F ] + \frac{20\kappa\sqrt{n_{y}}}{3NpT}\sum_{t=0}^{T-1}\sum_{n=1}^N\sum_{i=0}^{p-1} \mathbb{E}[ \| V_{t} - V_{t,i}^{(n)}\|_F ]  \ .
    \end{align}
\end{lemma}

\begin{lemma}\label{lemma_muon:global_grad_error_x_y}
    Given Assumptions~\ref{assumption:smooth}-\ref{assumption:ht_variance}, the gradient error regarding variable $x$ is bounded as:
    \begin{align}
        \frac{1}{T}\sum_{t=0}^{T-1}\mathbb{E}[ \| \nabla_x f(X_{t}, Y_{t}) - U_{t}\| ] & \leq \frac{(\eta_{x}\sqrt{n_x}+\eta_{y}\sqrt{n_y})pL_f}{\beta_{x} T} +\frac{1}{\beta_{x} T}\frac{2\sqrt{2}\sigma }{(Np)^{1-1/s}}  + \frac{(\gamma_{x}\sqrt{n_x}+\gamma_{y}\sqrt{n_y})L_f}{\beta_{x}} \notag \\
        & \quad +  (\eta_{x}\sqrt{n_x}+\eta_{y}\sqrt{n_y}) pL_f + \frac{2\sqrt{2}\beta_{x}^{1-1/s}}{(Np)^{1-1/s}}\sigma \ ,
    \end{align}
    the gradient error regarding variable $y$ is bounded as:
    \begin{align}
        \frac{1}{T}\sum_{t=0}^{T-1}\mathbb{E}[ \| \nabla_Y f(X_{t}, Y_{t}) - V_{t}\| ] & \leq \frac{(\eta_{x}\sqrt{n_x}+\eta_{y}\sqrt{n_y})pL_f}{\beta_{y} T} +\frac{1}{\beta_{y} T}\frac{2\sqrt{2}\sigma }{(Np)^{1-1/s}}  + \frac{(\gamma_{x}\sqrt{n_x}+\gamma_{y}\sqrt{n_y})L_f}{\beta_{y}} \notag \\
        & \quad +  (\eta_{x}\sqrt{n_x}+\eta_{y}\sqrt{n_y}) pL_f + \frac{2\sqrt{2}\beta_{y}^{1-1/s}}{(Np)^{1-1/s}}\sigma \ .
    \end{align}
\end{lemma}

\begin{lemma}\label{lemma_muon:local_grad_error_x_y}
    Given Assumptions~\ref{assumption:smooth}-\ref{assumption:ht_variance}, the consensus error on momentum regarding variable $x$ is bounded as:
    \begin{align}
        \frac{1}{NpT}\sum_{t=0}^{T-1}\sum_{n=1}^N\sum_{i=0}^{p-1} \mathbb{E}[ \| U_{t} - U_{t,i}^{(n)}\| ] & \leq 8\sqrt{2}\beta_{x}\sigma + 4\beta_{x}( \eta_{x}\sqrt{n_x}+\eta_{y}\sqrt{n_y}) pL_f +  2\beta_{x}(\gamma_{x}\sqrt{n_x}+\gamma_{y}\sqrt{n_y})L_f  \ ,
    \end{align}
    the consensus error on momentum regarding variable $y$ is bounded as:
    \begin{align}
        \frac{1}{NpT}\sum_{t=0}^{T-1}\sum_{n=1}^N\sum_{i=0}^{p-1} \mathbb{E}[ \| V_{t} - V_{t,i}^{(n)}\| ] & \leq 8\sqrt{2}\beta_{y}\sigma + 4\beta_{y}( \eta_{x}\sqrt{n_x}+\eta_{y}\sqrt{n_y}) pL_f +  2\beta_{y}(\gamma_{x}\sqrt{n_x}+\gamma_{y}\sqrt{n_y})L_f     \ .
    \end{align}
\end{lemma}

\subsection{Proof of the Theorem}
\begin{proof}
    We have established an upper bound for $\frac{1}{T}\sum_{t=0}^{T-1}\mathbb{E}[\| \nabla \Phi({X}_{t})  \|] $ in Eq.~(\ref{eq_muon:phi_sum}) as shown in Lemma~\ref{lemma_muon:phi_sum}. Next, we substitute the results from Lemma~\ref{lemma_muon:global_grad_error_x_y} and Lemma~\ref{lemma_muon:local_grad_error_x_y} into  Eq.~(\ref{eq:phi_sum}): 
    \begin{align}
        & \frac{1}{T}\sum_{t=0}^{T-1}\mathbb{E}[\| \nabla \Phi({X}_{t})  \|] \leq \frac{(\Phi(X_{0})- \Phi^*) }{\gamma_{x}T}  + \frac{\Phi(X_{0}) - f(X_0,Y_0)}{3\gamma_{x}T}  + \frac{2L_{\Phi}\gamma_{x}n_x}{3}  + \frac{L_f\gamma_{x}}{6} (n_{x}+100\kappa^2n_{y}+20\kappa n_{x}) \notag \\
        & + \frac{10\sqrt{n_x}}{3}\Big(\frac{(\eta_{x}\sqrt{n_x}+\eta_{y}\sqrt{n_y})pL_f}{\beta_{x} T} +\frac{1}{\beta_{x} T}\frac{2\sqrt{2}\sigma }{(Np)^{1-1/s}}  + \frac{(\gamma_{x}\sqrt{n_x}+\gamma_{y}\sqrt{n_y})L_f}{\beta_{x}}  +  (\eta_{x}\sqrt{n_x}+\eta_{y}\sqrt{n_y}) pL_f + \frac{2\sqrt{2}\beta_{x}^{1-1/s}}{(Np)^{1-1/s}}\sigma \Big) \notag \\
        & +  \frac{20\kappa\sqrt{n_y}}{3}\Big( \frac{(\eta_{x}\sqrt{n_x}+\eta_{y}\sqrt{n_y})pL_f}{\beta_{y} T} +\frac{1}{\beta_{y} T}\frac{2\sqrt{2}\sigma }{(Np)^{1-1/s}}  + \frac{(\gamma_{x}\sqrt{n_x}+\gamma_{y}\sqrt{n_y})L_f}{\beta_{y}} +  (\eta_{x}\sqrt{n_x}+\eta_{y}\sqrt{n_y}) pL_f + \frac{2\sqrt{2}\beta_{y}^{1-1/s}}{(Np)^{1-1/s}}\sigma\Big) \notag \\
        &  + \frac{10\sqrt{n_x}}{3}\Big( 8\sqrt{2}\beta_{x}\sigma + 4\beta_{x}( \eta_{x}\sqrt{n_x}+\eta_{y}\sqrt{n_y}) pL_f +  2\beta_{x}(\gamma_{x}\sqrt{n_x}+\gamma_{y}\sqrt{n_y})L_f    \Big) \notag \\
        & + \frac{10\kappa}{3}\Big(8\sqrt{2}\beta_{y}\sigma + 4\beta_{y}( \eta_{x}\sqrt{n_x}+\eta_{y}\sqrt{n_y}) pL_f  +  2\beta_{y}(\gamma_{x}\sqrt{n_x}+\gamma_{y}\sqrt{n_y})L_f    \Big) \ . 
    \end{align}
    Note that $n_x$ and $n_y$ are fixed constants, so all bounds remain of the same order, and the convergence rate is unchanged compared to Theorem~\ref{theorem:convergence-rate}.
    Therefore, we obtain
    \begin{align}
        \frac{1}{T}\sum_{t=0}^{T-1}\mathbb{E}[\| \nabla \Phi({X}_{t})  \|] \leq O\left(\frac{\kappa}{(TNp)^{1/4}}\right) + O\left(\frac{\kappa\sigma}{(TNp)^{\frac{s-1}{2s}}}\right) \ . 
    \end{align}
\end{proof}

\end{document}